\documentclass{article}

\usepackage[final,nonatbib]{neurips_2021}

\usepackage[utf8]{inputenc} 
\usepackage[T1]{fontenc}    
\usepackage{url}            
\usepackage{booktabs}       
\usepackage{amsfonts}       
\usepackage{nicefrac}       
\usepackage{microtype}      
\usepackage{xcolor}         

\usepackage{subcaption}          
\usepackage{alphalph} 
\usepackage{float}
\usepackage{caption}
\usepackage{verbatim}

\usepackage{algorithm,algorithmicx}
\usepackage{algpseudocode}
\usepackage{amsthm}
\usepackage{enumerate}
\usepackage{pgfplots}

\pgfplotsset{compat=1.9}
    \pgfplotsset{
        discard if not/.style 2 args={
            filter discard warning=false,
            x filter/.append code={
                \edef\tempa{\thisrow{#1}}
                \edef\tempb{#2}
                \ifx\tempa\tempb
                \else
                    
                \fi
            },
        },
    }

\usepgfplotslibrary{fillbetween}

\definecolor{borange}{RGB}{239,91,12}
\definecolor{bblue}{RGB}{0,56,101}
\usepackage[hidelinks,backref=page]{hyperref}        
\hypersetup{
    colorlinks=true,
    linkcolor=blue,
    filecolor=bmblue,      
    urlcolor=bmblue,
    citecolor=borange,
    pdftitle={Scaling Up Exact Neural Network Compression by ReLU Stability},
}

\newtheorem{definition}{Definition}
\newtheorem{lemma}{Lemma}

\newtheorem{corollary}[lemma]{Corollary}

\newtheorem{proposition}[lemma]{Proposition}

\usepackage{amsmath}
\usepackage{amsfonts}
\usepackage{bm}

\newcommand{\batch}{\mathcal{B}}
\newcommand{\meanbatch}{\mu_\batch}
\newcommand{\varbatch}{\sigma^2_\batch}

\def\1{\bm{1}}

\def\vb{{\bm{b}}}

\def\vf{{\bm{f}}}

\def\vw{{\bm{w}}}
\def\vx{{\bm{x}}}
\def\vy{{\bm{y}}}
\def\vz{{\bm{z}}}

\def\vP{{\bm{P}}}
\def\vQ{{\bm{Q}}}

\def\mI{{\bm{I}}}

\def\mW{{\bm{W}}}

\DeclareMathAlphabet{\mathsfit}{\encodingdefault}{\sfdefault}{m}{sl}
\SetMathAlphabet{\mathsfit}{bold}{\encodingdefault}{\sfdefault}{bx}{n}

\def\sD{{\mathbb{D}}}

\def\sL{{\mathbb{L}}}

\def\sP{{\mathbb{P}}}
\def\sQ{{\mathbb{Q}}}

\def\sS{{\mathbb{S}}}
\def\sT{{\mathbb{T}}}
\def\sU{{\mathbb{U}}}

\def\sX{{\mathbb{X}}}

\newcommand{\lRegularization}{\ell}
\newcommand{\lOne}{\lRegularization_1}
\newcommand{\lTwo}{\lRegularization_2}

\newcommand{\myReferTable}[1]{Tab.~\ref{#1}}
\newcommand{\myReferFigure}[1]{Fig.~\ref{#1}}

\definecolor{bblue}{rgb}{0,0.15,0.25}
\definecolor{bmblue}{rgb}{0,0.51,0.73}
\definecolor{borange}{rgb}{0.97,0.3,0.09}

\definecolor{mpl_red}{rgb}{0.9,0,0}
\definecolor{mpl_orange}{rgb}{0.98,0.45,0.02}
\definecolor{mpl_green}{rgb}{0.08,0.69,0.10}
\definecolor{mpl_blue}{rgb}{0.01,0.26,0.87}
\definecolor{mpl_violet}{rgb}{0.6,0.5,0.92}

\newcommand{\accuracyCompressionPlotWithLabels}[9]{
        \begin{tikzpicture}[scale={#1}]
        \begin{axis}[axis y line*=left,
            cycle list name=mark list,
            ymin={#2},
            ymax={#3},
            xtick={#6},
            xlabel={\Huge{$\ell_1$}},
            ylabel={\Huge{\color{bmblue} Test accuracy (\%)}},
            xmajorgrids=true,
            ticklabel style={font=\huge},
            ]
            \addplot+[bmblue, smooth,
                line width={#4},
                scatter,
                mark=none]
            table[x=reg,y=acc]
            {#5};
            \filldraw [fill=#7, draw=none, opacity=0.1] (#8,0) rectangle (#9,1000);
        \end{axis}
        \begin{axis}[axis y line=right,
            axis x line=none,
            ymin=-2, ymax=62,
            ytick={0,30,60},
            ylabel style={align=center},
            ylabel={\Huge{\color{borange} Nodes removed (\%)}},
            ymajorgrids=true,
            ticklabel style={font=\huge},
            ]
            \addplot+[borange, smooth,
                line width={#4},
                scatter,
                mark=none,]
            table[x=reg,y=nodes-percent]
            {#5};
        \end{axis}
        \end{tikzpicture}
        }

\newcommand{\accuracyCompressionPlot}[9]{\begin{tikzpicture}[scale={#1}]
                \begin{axis}[axis y line*=left,
                    cycle list name=mark list,
                    ylabel={{\fontsize{50}{60}\selectfont \phantom{-}}},
                    ymin={#2},
                    ymax={#3},
                    xtick={#6},
                    xtick style={draw=none},
                    xmajorgrids=true,
                    xticklabels={,,},
                    yticklabels={,,}
                    ticklabel style={font=\tiny},
                    scaled x ticks=false
                    ]
                    \addplot+[bmblue, smooth,
                        line width={#4},
                        scatter,
                        mark=none]
                    table[x=reg,y=acc]
                    {#5};
                    \filldraw [fill=#7, draw=none, opacity=0.1] (#8,0) rectangle (#9,1000);
                \end{axis}
                \begin{axis}[axis y line=right,
                    axis x line=none,
                    ymin=-2, ymax=62,
                    ytick={0,30,60},
                    xtick style={draw=none},
                    ymajorgrids=true,
                    xticklabels={,,},
                    yticklabels={,,},
                    ticklabel style={font=\tiny},
                    ]
                    \addplot+[borange, smooth,
                        line width={#4},
                        scatter,
                        mark=none,]
                    table[x=reg,y=nodes-percent]
                    {#5};
                \end{axis}
            \end{tikzpicture}
            }

\newcommand{\runtimeRegularizationPlot}[5]{
    \begin{tikzpicture}[scale={#1}]
        \begin{semilogyaxis}
            [
            cycle list name=mark list,
            ymin={#2},
            ymax={#3},
            xtick={#4},
            xlabel={\Large{$\ell_1$}},
            ylabel={\Large{Runtime (s)}},
            xmajorgrids=true,
            ymajorgrids=true,
            ticklabel style={font=\Large},
            legend style={nodes={scale=0.75, inner sep=0.5pt, transform shape}, inner xsep=1pt, inner ysep=0.5pt, at={(0.5,1.43)},anchor=north east},
            ]
            {#5}
        \end{semilogyaxis}
    \end{tikzpicture}
}

\newcommand{\addLinePlotWithErrorBar}[7]{
    \addplot[name path=upper,draw=none,scatter, mark=none, forget plot] table[x=reg,y expr=\thisrow{#6}+ 0.5*\thisrow{#7},discard if not={architecture}{#4}] {#3};
    \addplot[name path=lower,draw=none,scatter, mark=none, forget plot] table[x=reg,y expr=\thisrow{#6}- 0.5*\thisrow{#7},discard if not={architecture}{#4}] {#3};
    \addplot[fill={#1}, fill opacity=0.5, mark=none, forget plot] fill between[of=upper and lower];

    \addplot+[{#1}, {#5},
        line width={#2},
        scatter,
        mark=none]
    table[x=reg,y={#6},discard if not={architecture}{#4}]
    {#3};

}

\title{Scaling Up Exact Neural Network Compression by ReLU Stability}

\author{%
  Thiago Serra \\
  Bucknell University \\
  Lewisburg, PA, United States \\
  \texttt{thiago.serra@bucknell.edu} \\
  \And
  Xin Yu \\
  University of Utah \\
  Salt Lake City, UT, United States \\
  \texttt{xin.yu@utah.edu} \\
  \AND
  Abhinav Kumar \\
  Michigan State University \\
  East Lansing, MI, United States \\
  \texttt{kumarab6@msu.edu} \\
  \And
  Srikumar Ramalingam \\
  Google Research \\
  New York, NY, United States \\
  \texttt{rsrikumar@google.com} \\
}

\begin{document}

\maketitle

\begin{abstract}
We can compress a rectifier network while exactly preserving its underlying functionality with respect to a given input domain if some of its neurons are stable. 
However, current approaches to determine the stability of neurons with Rectified Linear Unit~(ReLU) activations require solving or finding a good approximation to multiple discrete optimization problems. 
In this work, 
we introduce an algorithm based on solving a single optimization problem to identify all stable neurons.  
Our approach is on median 183 times faster than the state-of-art method on CIFAR-10, which 
allows us to explore exact compression on deeper ($5 \times 100$) and wider ($2 \times 800$) networks 
within minutes. 
For classifiers trained under an amount of $\lOne$ regularization that does not worsen accuracy, we can remove up to 56\% of the connections on the CIFAR-10 dataset. The code is available at the following link,
\url{https://github.com/yuxwind/ExactCompression}.
\end{abstract}

\section{Introduction}\label{sec:introduction}

For the past decade, 
the computing requirements associated with state-of-art machine learning models 
have grown faster 
than typical hardware improvements~\cite{amodei2018aicompute}.  
Although those requirements are often associated with training neural networks, 
they also translate into larger models, which are challenging to deploy 
in modest computational environments,  
such as in mobile devices. 

Meanwhile, we have learned that the expressiveness of the models associated with neural networks---when measured in terms of their number of linear regions
---grows polynomially on the number of neurons and occasionally exponentially on the network depth~\cite{pascanu2013regions,montufar2014linear,raghu2017expressive,serra2018bounding,hanin2019complexity,hanin2019deep}. 
Hence, we may wonder if 
the pressing need for larger models could not be countered by such gains in model complexity. 
Namely, if we could not represent the same model using a smaller neural network. 
More specifically, 
we consider the following definition of equivalence~\cite{kumar2019transformations, serra2020lossless}:  

\begin{definition}
Two neural networks $\mathcal{N}_1$ and $\mathcal{N}_2$ with associated functions $\vf_1 : \mathbb{R}^{n_0} \rightarrow \mathbb{R}^m$ and $\vf_2 : \mathbb{R}^{n_0} \rightarrow \mathbb{R}^m$ are local equivalent with respect to a domain $\sD \subseteq \mathbb{R}^{n_0}$ if $\vf_1(x) = \vf_2(x)~\forall x \in \sD$.
\end{definition}

There is an extensive literature on methods for compressing neural networks~\cite{cheng2018survey,blalock2020survey},  
which is aimed at obtaining smaller networks that are nearly as good as the original ones. 
These methods generally produce networks that are not equivalent, 
and thus require retraining the neural network for better accuracy. They may also lead to models in which the relative accuracy for some classes is more affected than that of other classes~\cite{hooker2019forget}.

Compressing a neural network while preserving its associated function is a relatively less explored topic, 
which has been commonly referred to as \emph{lossless compression}~\cite{serra2020lossless,sourek2021lossless}. 
However, that term has also been used for the more general case in which the overall accuracy of the compressed network is no worse than that of the original network
regardless of equivalence~\cite{xing2020lossless}. 
Hence, we regard \emph{exact compression} as a more appropriate term when equivalence is preserved.

Exact compression has distinct benefits and challenges. 
On the one hand, there is no need for retraining and no risk of disproportionately affecting some classes more than others. 
On the other hand, optimization problems that are formulated for exact compression need to account for any valid input as opposed to relying on a sample of inputs. 
In this paper, we focus on how to scale such an approach 
to a point in which exact compression starts to become practical for certain applications.

In particular, 
we introduce and evaluate a faster algorithm for exact compression based on identifying all neurons with Rectified Linear Unit~(ReLU) activation that have linear behavior, 
which are denoted as \emph{stable}. 
In other words, those are the neurons for which the mapping of inputs to outputs is always  characterized by a linear function, 
which is either the constant value $0$ or the  preactivation output. 
We can remove or merge such neurons---and even  entire layers in some cases---while obtaining a smaller but equivalent neural network. 
Our main contributions are the following: 

\begin{enumerate}[(i)]
\item \label{itm:first} We propose the algorithm \texttt{ISA} (Identifying Stable Activations), which is based on solving a single Mixed-Integer Linear Programming~(MILP) formulation to verify the stability of all neurons of a 
feedforward neural network with ReLU activations. \texttt{ISA} is  
faster than 
solving MILP formulations for every neuron---either optimally~\cite{tjeng2019stability} or approximately~\cite{serra2020lossless}. 
Compared to~\cite{serra2020lossless}, the median improvement is of $\mathbf{83}$ times on MNIST dataset ($\mathbf{183}$ times on CIFAR-10 dataset and $\mathbf{137}$ times on CIFAR-100 dataset)~---and in fact greater in larger networks.

\item We reduce the runtime with a GPU-based preprocessing step that identifies neurons that are not stable with respect to the training set. 
The median improvement for that part alone is of $\mathbf{3.2}$ times on MNIST dataset. 

\item We outline and prove the correctness of a new compression algorithm, \texttt{LEO++} (Lossless Expressiveness Optimization, as in~\cite{serra2020lossless}),  which leverages \textcolor{blue}{(}\ref{itm:first}\textcolor{blue}{)} to perform all compressions once per layer instead of once per stable neuron~\cite{serra2020lossless}.

\item We leverage the scalability of our approach to investigate exact compressibility 
on classifiers that are deeper ($5 \times 100$) and wider ($2 \times 800$) than previously studied in~\cite{serra2020lossless} ($2 \times 100$).  
We show that approximately $20$\% of the neurons and $40$\% of the connections can be removed from MNIST classifiers trained with an amount of $\lOne$ regularization that does not worsen accuracy.
\end{enumerate}

\section{Related work}\label{sec:related}

There are many pruning methods for sparsifying or reducing the size of neural networks by 
removing connections, neurons, or even 
layers. 
They are justified by the significant redundancy among  parameters~\cite{denil2013parameters} and 
the better generalization bounds of compressed networks~\cite{arora2018bounds,zhou2019bounds,suzuki2020spectral,suzuki2020bounds}.

Surveys such as~\cite{blalock2020survey} note that these methods are typically based on a tradeoff between model efficiency and quality: 
the models of compressed neural networks tend to have a comparatively lower accuracy, 
save some exceptions~\cite{han2015connections,xing2020lossless,suzuki2020spectral}. 
Nevertheless, 
the loss in accuracy due to compression is disproportionately distributed across classes and more severe in a fraction of them; 
the most impacted inputs are those that the original network could not classify well; 
and the overall robustness to noise or adversarial examples is diminished~\cite{hooker2019forget}.
Furthermore,
the amount of compression yielding similar performance to the original network can vary significantly depending on the task~\cite{lost}. 

To make up for model changes and potential accuracy loss, one may rely on a three-step procedure consisting of 
(1) training the neural network; 
(2) compression; and 
(3) retraining. 
Nevertheless, the scope of compression methods is seldom restricted to the second step. 
For example, 
the compressibility of a neural network hinges on how it was trained, 
with regularizations such as $\lOne$ and $\lTwo$ often used to make part of the network parameters negligible in magnitude---and hopefully in impact as well. 

In fact, the two main---and recurring---themes in this topic are pruning connections 
when the corresponding parameters are sufficiently small~\cite{hanson1988minimal,mozer1989relevance,janowsky1989prunning,han2015connections,han2016deepcompression,li2017cnn,frankle2019lottery,gordon2020bert,tanaka2020synapticflow} and when the impact of the connection on the loss function is sufficiently small~\cite{lecun1989damage,hassibi1993surgeon,lebedev2016braindamage,molchanov2017taylor,dong2017layerwise,yu2018importance,zeng2018mlprune,lee2019pretraining,wang2019eigendamage,wang2020tickets,woodfisher}. 
The main issue with the first approach is 
that small weights may nevertheless be important, 
although it is possible to empirically quantify their impact on the loss function~\cite{rosenfeld2020predictability}.
The main issue with the second approach is the 
computational cost of calculating the second-order derivatives of the loss function in deep networks,  
which has lead to many approaches for approximating such values.

Overlapping with such approximations, 
there is a growing literature on casting neural network compression as an optimization problem~\cite{he2017cnn,luo2017thinet,aghasi2017nettrim,yu2018importance,serra2020lossless,elaraby2020importance}. 
Most often, 
these formulations aim to minimize the impact of the compression on how the neural network performs on the training set.

Other lines of work and overlapping themes in neural network compression include 
combining similar neurons~\cite{srinivas2015datafree,mariet2016diversity,suau2020distillation,suzuki2020spectral}; 
low-rank approximation, factorization, 
and random projection of the weight matrices~\cite{jaderberg2014lowrank,denton2014linear,lebedev2015decomposition,arora2018bounds,wang2018wide,su2018tensorial,wang2019eigendamage,suzuki2020spectral,suzuki2020bounds,li2020understanding}; and statistical tests on the relevance of a connection to network output~\cite{xing2020lossless}.
Many recent approaches focus on pruning 
at initialization instead of after training~\cite{lee2019pretraining,lee2020pretraining,wang2020tickets,tanaka2020synapticflow,frankle2021initialization} 
as well as on what parameters to use when these networks are retrained~\cite{frankle2019lottery,liu2019structure,renda2020retraining}.

Exact compression was only recently explored for fully-connected feedforward neural networks~\cite{serra2020lossless} and graph neural networks~\cite{sourek2021lossless}. 
Nevertheless, 
we may associate it with the literature on neural network equivalency, which includes  
verifying that  networks are equivalent~\cite{narodytska2018verifying,buning2020equivalence}, 
identifying operations that produce equivalent networks~\cite{hechtnielsen1990weightspaces,chen1993geometry,kurkova1994equivalent,kumar2019transformations,phuong2020equivalence}, 
reconstructing networks from their outputs~\cite{albertini1993form,albertini1993identifiability,fefferman1994recovering,alfalou2003recurrent,rolnick2020engineering},  
and evaluating the effect of redundant representations on  training~\cite{berner2019degeneracy,petersen2020topological}.

\section{Setting and notation}\label{sec:notation}

We consider fully-connected feedforward neural networks 
with $L$ hidden layers, in which we denote $n_l$ as the number of units---or width---of layer $l \in \sL := \{1, 2, \ldots, L\}$ and $x^l_i$ as the output of the $i$-th unit of layer $l$ for $i \in \{1, 2, \ldots, n_l\}$. 
For uniformity, we denote $\vx^0 \in \mathbb{R}^{n_0}$ as the network input. 
We denote the output of the $i$-th unit of layer $l$ as $x^l_i = \sigma^l(y^l_i)$, 
where the pre-activation output $y^l_i := \vw^l_i \cdot  x^{l-1} + b^l_i$ is defined by the learned weights $\vw^l_i \in \mathbb{R}^{n_{l-1}}$ and the bias $b^l_i \in \mathbb{R}$ of the unit 
as well as the activation function $\sigma^l : \mathbb{R} \rightarrow \mathbb{R}$ associated with layer $l$, 
which is $\sigma^l(u) = \max\{0,u\}$---the ReLU~\cite{hahnloser2000relu}. 
The output layer may have a different structure, such as the softmax layer~\cite{bridle1990softmax}, which is nevertheless irrelevant 
for our purpose of compressing the hidden layers. 
For every layer $l \in \sL$, 
let $\mW^l = [ \vw^l_1 \vw^l_2 \ldots \vw^l_{n_l} ]^T$ be the matrix of weights, 
$\mW^l_{\sS}$ be a submatrix of $\mW^l$ consisting of the rows in set $\sS$,  
and $\vb^l = [ b^l_1 b^l_2 \ldots b^l_{n_l} ]^T$ be the vector of biases. 
Finally, let
$\mI_m(\sS)$ be an $m \times m$ diagonal matrix in which the $i$-th diagonal element is 1 if $i \in \sS$ and 0 if $i \notin \sS$.

\section{Identifying stability for exact compression}\label{sec:stability}

This section explains the concept of stability and describes how MILP has been used to identify stable neurons.
If the output of neuron $i$ in layer $l$ is always linear on its inputs, we say that the neuron is stable. 
This happens in two ways for the ReLU activation.
When $x^l_i = 0$ for any valid input, which implies that $y^l_i \leq 0$, 
we say that the neuron is \emph{stably inactive}. 
When $x^l_i = y^l_i$ for any valid input, which implies that $y^l_i \geq 0$, 
we say that the neuron is \emph{stably active}. 

The qualifier \emph{valid} is essential since not every input may occur in practice. 
If $\vw^l_i \neq \bm{0}$, 
there are nonempty halfspaces on $\vx^{l-1}$ that would make 
that neuron 
active or inactive, $\{ \vx^{l-1} : \vw^l_i \cdot \vx^{l-1} + b^l_i \leq 0 \}$ and $\{ \vx^{l-1} : \vw^l_i \cdot \vx^{l-1} + b^l_i \geq 0 \}$, but it is possible that valid inputs only map to one of them. 
For the first layer, 
we only need to account for the valid inputs to the neural network. 
For example, the domain of a network trained on the MNIST dataset is $\{ \vx^0 : \vx^0 \in [0,1]^{784} \}$~\cite{lecun1998mnist}. 
For the remaining hidden layers, 
we also account for the combinations of outputs that can be produced by the preceding layers given their valid inputs and  parameters. 
Hence, assessing stability is no longer straightforward. 

We can determine if a neuron of a trained neural network 
is stable by solving optimization problems to maximize and minimize its preactivation output~\cite{tjeng2019stability}.  
The main decision variables in these problems are the inputs for which the preactivation output is optimized. 
Consequently, there is also a decision variable associated with the output of every neuron, 
in addition to other variables described below.

\textbf{MILP formulation of a single neuron ~~} 
For every neuron $i$ of layer $l$, 
we map every input vector $\vx^{l-1}$ to the corresponding output $x^l_i$ through a set of linear constraints that also include 
a binary variable $z^l_i$ denoting if the unit is active or not, 
a variable for the pre-activation output $y^l_i$, 
a variable $\chi^l_i := \max\{0, -y^l_i\}$ denoting the output of a complementary fictitious unit, 
and positive constants $M^l_i$ and  $\mu^l_i$ that are as large as $x^l_i$ and $\chi^l_i$ can be. 
The formulation below is explained in Appendix~\ref{ap:mip_relu}. 
\begin{align}
    \vw^l_i \cdot \vx^{l-1} + b^l_i = y^l_i = x^l_i - \chi^l_i \label{cons:mip_relu_main_start} \\
    0 \leq x^l_i \leq M^l_i z^l_i  \\
    0 \leq \chi^l_i \leq \mu^l_i (1 - z^l_i)  \\
    z^l_i \in \{0, 1\} \label{cons:mip_relu_main_end}
\end{align}

\textbf{Using MILP to determine stability ~~}
Let $\sX \subset \mathbb{R}^{n_0}$ be the set of valid inputs for the neural network, 
which we may assume to be bounded in every direction. 
We can obtain the interval $[\underline{\mathcal{Y}}^{l'}_i, \overline{\mathcal{Y}}^{l'}_i]$ 
for the preactivation output $y^{l'}_i$ of neuron $i$ in layer $l'$ by solving the following optimization problems~\cite{tjeng2019stability}:
\begin{align}\label{eq:min_y}
    \underline{\mathcal{Y}}^{l'}_i := \left\{ \min  \vw^{l'}_i \cdot \vx^{l'-1} + b^{l'}_i  :  \vx^0 \in \sX; \eqref{cons:mip_relu_main_start}-\eqref{cons:mip_relu_main_end}~\forall l \in [l'-1], i \in [n_l] \right\}
\end{align}
\begin{align}\label{eq:max_y}
    \overline{\mathcal{Y}}^{l'}_i := \left\{ \max  \vw^{l'}_i \cdot \vx^{l'-1} + b^{l'}_i  :  \vx^0 \in \sX; \eqref{cons:mip_relu_main_start}-\eqref{cons:mip_relu_main_end}~\forall l \in [l'-1], i \in [n_l] \right\}
\end{align}
When $\overline{\mathcal{Y}}^{l'}_i \leq 0$, 
then $x^{l'}_i = 0$ for every $x^0 \in \sX$ and the neuron is stably inactive. 
When $\underline{\mathcal{Y}}^{l'}_i \geq 0$, 
then $x^{l'}_i = y^{l'}_i$ for every $x^0 \in \sX$ and the neuron is stably active. 

Variations of the formulations above have been proposed for diverse tasks over neural networks, 
such as verifying them~\cite{cheng2017mip}, embedding their model into a broader decision-making problem~\cite{sanner2017planning,bergman2019janos,delarue2020rlvrp}, 
and measuring their expressiveness~\cite{serra2018bounding}. 
When stable units are identified, other optimization problems over trained neural networks become easier to solve~\cite{tjeng2019stability}. 
For example, weight regularization can induce neuron stability and facilitate adversarial robustness verification~\cite{xiao2019training}. 
There is extensive work on the properties of such formulations and methods to solve them more effectively~\cite{fischetti2018mip,huchette2019ipco,botoeva2020dependency,serra2020empirical,huchette2020mp}.

For the purpose of identifying stable neurons, however, 
it is not scalable to analyze large neural networks by 
solving such optimization problems for every  neuron~\cite{tjeng2019stability}---or even by just approximately solving each of them to ensure that $\overline{\mathcal{Y}}^{l'}_i \leq 0$ or $\underline{\mathcal{Y}}^{l'}_i \geq 0$~\cite{serra2020lossless}.

\section{A new algorithm for exact compression}\label{sec:algorithm}

Based on observations discussed in what follows (I to~III), 
we propose a new MILP formulation to identify stable neurons (Section~\ref{sub:milp}), 
means to generate feasible solutions while the formulation is solve (Section~\ref{sub:feasible}), 
a preprocessing step to reduce the effort to solve the formulation (Section~\ref{sub:preprocess}), 
the resulting algorithm \texttt{ISA} for identifying all stable neurons at once (Section~\ref{sub:stability_algo}), 
and the revised compression algorithm \texttt{LEO++} exploiting all stable neurons in each layer at once (Appendix~\ref{ap:compression_algo}).

\subsection{A new MILP formulation}\label{sub:milp}

Consider the two observations below and their implications: 

\textbf{I: The overlap between optimization problems ~~}
Although previous approaches require solving many optimization problems, 
their formulations are all very similar: 
we maximize or minimize the same objective function for each neuron, 
the feasible set of the problems for each layer are the same, 
and they are contained in the feasible set of problems for the subsequent layers. 

\textbf{II: Proving stability is harder than disproving it ~~}
Certifying that a neuron is stable is considerably more complex than certifying that a neuron is \emph{not} stable. 
For the former, 
we need to exhaustively show that all inputs lead to the neuron always being active or always being inactive, 
which can be achieved by solving \eqref{eq:min_y} and \eqref{eq:max_y} for every neuron. 
For the latter, 
we just need a pair of inputs to the neural network such that the neuron is active with one of them and inactive with the other. 

Therefore, we consider the problem of finding an input that serves as a certificate of a neuron not being stable to as many neurons of unknown classification as possible. 
%
%
For that purpose, 
we define a decision variable $p^l_i \in \{0,1\}$ to denote if an input activates neuron $i$ in layer $l$. 
Likewise, 
we define a decision variable $q^l_i \in \{0,1\}$ to denote if an input does not activate neuron $i$ in layer $l$. 
Furthermore, 
we restrict the scope of the problem to state that have not been previously observed by using $\sP^l \subseteq \{1, \ldots, n_l\}$ as the set of neurons in layer $l$ for which there is no known input that activates the neuron. 
Likewise, 
we use $\sQ^l \subseteq \{1, \ldots, n_l\}$ as the set of neurons in layer $l$ for which there is no known input that does not activate the neuron. 
For brevity, let $\vP := (\sP^1, \ldots, \sP^L)$ and $\vQ := (\sQ^1, \ldots, \sQ^L)$ characterize an instance of such optimization problem, 
which is formulated as follows:
\begin{align}
    \mathcal{C}\left( \vP, \vQ \right) = & \max & \sum_{l \in \sL} \left( \sum_{i \in \sP^l} p^l_i + \sum_{i \in \sQ^l} q^l_i \right) \label{cons:of} \\
    &\text{s.t.} & \vx^0 \in \sX \\ 
         && \eqref{cons:mip_relu_main_start}-\eqref{cons:mip_relu_main_end}~\forall l \in \sL, i \in [n_l] \label{cons:pq_unit} \\
         && 0 \leq p^l_i \leq z^l_i~\forall l \in \sL, i \in \sP^l \label{cons:p} \\
         && 0 \leq q^l_i \leq 1-z^l_i~\forall l \in \sL, i \in \sQ^l \label{cons:q} \\
         && p^l_i, q^l_i \in \{0,1\} \label{cons:pq_bin}
\end{align}
Note that constraint~\eqref{cons:pq_bin} is actually not necessary. We refer to Appendix~\ref{ap:binary_relaxation} 
for more details.

The formulation above yields an input that maximizes the number of neurons with an activation state that has not been previously observed. 
The following results show that it entails an approach in which no more than $N+1$ such formulations are solved. We refer to Appendix~\ref{ap:milp_proofs} for the proofs.

\begin{proposition}\label{prop:cpq0}
If $\mathcal{C}(\vP,\vQ)=0$, 
then every neuron $i \in \sP^l$ is stably inactive 
and every neuron $i \in \sQ^l$ is stably active. 
\end{proposition}

\begin{corollary}\label{cor:np1}
The stability of all neurons of a neural network can be determined by solving formulation~\eqref{cons:of}--\eqref{cons:pq_bin} at most $N+1$ times, 
where $N := \sum\limits_{l \in \sL} n_l$.
\end{corollary}

Those results imply that  
we can iteratively solve the new formulation as part of an algorithm to identify all stable neurons.
In fact, we can determine the stability of the entire neural network with a single call to the MILP solver. 
Except for the last time that formulation \eqref{cons:of}--\eqref{cons:pq_bin} is solved, 
there is no need to solve it to optimality: 
any solution with a positive objective function value can be used to reduce the number of unobserved states. 
Hence, all that we need is a way to inspect every feasible solution obtained by the MILP solver and then remove the solutions in which either $p^l_i=1$ or $q^l_i=1$ for states that were already observed. 
Both of those needs can be addressed in fully-fledged MILP solvers by implementing a lazy constraint callback. 
We refer to Appendix~\ref{ap:callback} 
for more details. 
When we finally reach $\mathcal{C}(\vP,\vQ)=0$, 
the correctness of the MILP solver serves as a certificate of the stability of those remaining neurons.
The resulting algorithm is described in Section~\ref{sub:stability_algo}.

\subsection{Inducing feasible MILP solutions}\label{sub:feasible}

The runtime with a single solver call depends on 
the frequency with which feasible solutions are obtained. 
Although at most $N+1$ optimal solutions would suffice if we were to make consecutive calls to the solver until $\mathcal{C}(\vP, \vQ) = 0$, 
we should not expect the same from the first $N+1$ feasible solutions found by the MILP solver while using the lazy constraint callback because 
they may not have a positive objective function value 
due to the $p^l_i$ and $q^l_i$ variables that 
have been fixed to 0. 
On top of that, 
obtaining a feasible solution for an MILP formulation is NP-complete~\cite{cook1971feasibility}. 

\textbf{III: Finding feasible solutions to MILP formulations of neural networks is easy ~~}
To any valid input of the neural network there is a corresponding solution of the MILP formulation: the neural network input implies which neurons are active and what is their output when active~\cite{fischetti2018mip}. 

Although any random input would suffice, we have found that it is better in practice to use inputs indirectly generated by the MILP solver. 
Namely, we can use the solution of the Linear Programming~(LP) relaxation, which is solved at least once per branch-and-bound node.  
The LP relaxation is obtained from the MILP formulation by relaxing its integrality constraints. 
In the case of binary variables with domain $\{0, 1\}$, 
that consists of relaxing the domain of such variables to the continuous interval $[0,1]$. 
We use the values of $\vx^0$ in the solution of the LP relaxation as the network input, 
and thus obtain a feasible MILP solution by replacing the values of the other variables---which may be fractional for the decision variables with binary domains---by the values implied by fixing $\vx^0$. 
However imprecise due to the relaxation of the binary domains, 
the input defined by the optimal solution of the LP relaxation may intuitively guide us toward maximizing the objective function. 

\subsection{Preprocessing}\label{sub:preprocess}

In addition to generating feasible MILP solutions at every node of the branch-and-bound tree during the solving process, 
we also evaluate the training set on the trained neural network to reduce the number of states that need to be search for by the MILP solver. 
By using GPUs, this step can be completed in few seconds for all the experiments performed.

\subsection{Identifying stable neurons}\label{sub:stability_algo}

Algorithm~\ref{alg:stability}, which we denote \texttt{ISA} (Identifying Stable Activations), identifies all stable neurons of a neural network. 
The prior discussion on identifying stable units leads to the steps described between lines~\ref{alg:b_initialization} and \ref{alg:e_optimization}. 
First, 
$\vP$ and $\vQ$ are initialized between  lines~\ref{alg:b_initialization} and \ref{alg:e_initialization} according to the preprocessing step described above. 
Next, 
the MILP formulation is iteratively solved between lines~\ref{alg:b_optimization} and \ref{alg:e_optimization}. 
The block between lines~\ref{alg:b_termination} and \ref{alg:e_termination} identifies the termination criterion, 
which implies that the unobserved states cannot be obtained with any valid input. 
The block between lines~\ref{alg:b_feasible} and  \ref{alg:e_feasible} inspects every feasible solution to identify unobserved states and then to effectively remove the decision variables associated with those states from the objective function by adding a constraint that sets their value to 0. 
The block between lines~\ref{alg:b_relaxed} and  \ref{alg:e_relaxed} produces a feasible solution from a solution of the LP relaxation when the latter is produced by the MILP solver. 
For brevity, 
we assume that the block between lines~\ref{alg:b_feasible} and  \ref{alg:e_feasible} would leverage such solution at the next repetition of the loop. 

\begin{algorithm}[h!]
   \caption{\texttt{ISA} provably identifies all stable neurons of a neural network by iterating over the solution of a single MILP formulation to verify the occurrence of states unobserved in the training set}
   \label{alg:stability}
\begin{algorithmic}[1]
   \State {\bfseries Input:} neural network $\left(L, \left\{ (n_l, \mW^l, \vb^l) \right\}_{l \in \sL}\right)$ 
   \State {\bfseries Output:} stable neurons  $\left(\left\{ (\sP^l, \sQ^l) \right\}_{l \in \sL}\right)$
   \For{$l \gets 1$ {\bfseries to} $L$}  \Comment{Pre-processing step} \label{alg:b_initialization} 
     \State $\sP^l \gets$ subset of $\{1, \ldots, n_l\}$ that is \emph{never} activated by the training set 
     \State $\sQ^l \gets$ subset of $\{1, \ldots, n_l\}$ that is \emph{always} activated by the training set 
   \EndFor \label{alg:e_initialization}
   \While {solving $\mathcal{C}(\vP,\vQ)$}
   \Comment{Loop interacting with MILP solver}
   \label{alg:b_optimization} 
     \If{optimal value is proven to be 0} \label{alg:b_termination}
     \Comment{Remaining neurons are all stable}
       \State {\bfseries break} \label{alg:e_termination}
       \Comment{Nothing else to be done}
     \ElsIf{found positive MILP solution $(\bar{x},\bar{z},\bar{p},\bar{q})$}
     \Comment{Identified unobserved states}
     \label{alg:b_feasible}
       \For{$l \gets 1$ {\bfseries to} $L$} 
       \Comment{Loops over all hidden layers}
         \For{{\bfseries every} $i \in \sP^l$}
         \Comment{Loops over neurons that have not been seen \emph{active} yet}
            \If{$\bar{p}^l_i > 0$}
            \Comment{Neuron is active for the first time}
                \State $\sP^l \gets \sP^l \setminus \{ i \}$
                \Comment{Neuron is not stably inactive}
                \State $p^l_i \gets 0$ 
                \Comment{Restricts MILP to avoid identifying neuron again}
            \EndIf
         \EndFor
         \For{{\bfseries every} $i \in \sQ^l$}
         \Comment{Loops over neurons that have not been seen \emph{inactive} yet}
            \If{$\bar{q}^l_i > 0$}
            \Comment{Neuron is inactive for the first time}
                \State $\sQ^l \gets \sQ^l \setminus \{ i \}$
                \Comment{Neuron is not stably active}
                \State $q^l_i \gets 0$
                \Comment{Restricts MILP to avoid identifying neuron again}
            \EndIf
         \EndFor
       \EndFor \label{alg:e_feasible}
     \ElsIf{found LP relaxation solution $(\tilde{x},\tilde{z},\tilde{p},\tilde{q})$}
     \Comment{Input $\tilde{x}$ \emph{may} produce unseen states}
     \label{alg:b_relaxed}
       \State {use $\tilde{x}^0$ to produce an MILP solution $(\bar{x},\bar{z},\bar{p},\bar{q})$}
       \Comment{Produce unseen activations for input}
       \label{alg:e_relaxed}
     \EndIf
   \EndWhile \label{alg:e_optimization}
   \State \Return $\left(\left\{ (\sP^l, \sQ^l) \right\}_{l \in \sL}\right)$
\end{algorithmic}
\end{algorithm}

\section{Experimental results}\label{sec:experiments}

We trained and evaluated the compressibility of classifiers for the datasets MNIST~\cite{lecun1998mnist}, CIFAR-10~\cite{krizhevsky2009learning}, and CIFAR-100~\cite{krizhevsky2009learning} with and without $\lOne$ weight regularization, 
which is known to induce stability~\cite{tjeng2019stability}. 
We refer 
to Appendix~\ref{ap:implementation} 
for details on environment and implementation. 
We use the notation $L \times n$ for the architecture of $L$ hidden layers with $n$ neurons each. 
We started at $L=2$ and $n=100$, and then doubled the width $n$ or incremented the depth $L$ until the majority of the runs for MNIST classifiers for any configuration timed out after $3$ hours. 
With preliminary runs, we chose values for $\lOne$ which spanned from those for which accuracy is improving as $\lOne$ increases until those for which the accuracy starts decreasing. 
We trained and evaluated neural networks with 5 different random initialization seeds for each choice of $\lOne$. 
The amount of regularization used did not stabilize the entire layer. 
We refer to Appendix~\ref{ap:experiments} for additional figures and tables with complete results. 

\textbf{Regularization and compression ~~}
\myReferFigure{fig:curves} illustrates the average accuracy and number nodes that can be removed from networks according to architecture and dataset based on the amount of regularization used. 
When used in moderate amounts, regularization improves accuracy and very often that also coincides with enabling exact compression. We observe regularization improving accuracy in $17$ of the $21$ plots in \myReferFigure{fig:curves}. In $14$ cases, we reduce the size of these more accurate networks. Those include all architectures for MNIST (a to g), the architecture with width $400$ for CIFAR-100 (r), and all the architectures with more hidden layers for CIFAR-10 and CIFAR-100 (j, l, n, q, s, u).

\begin{figure}[!tb]
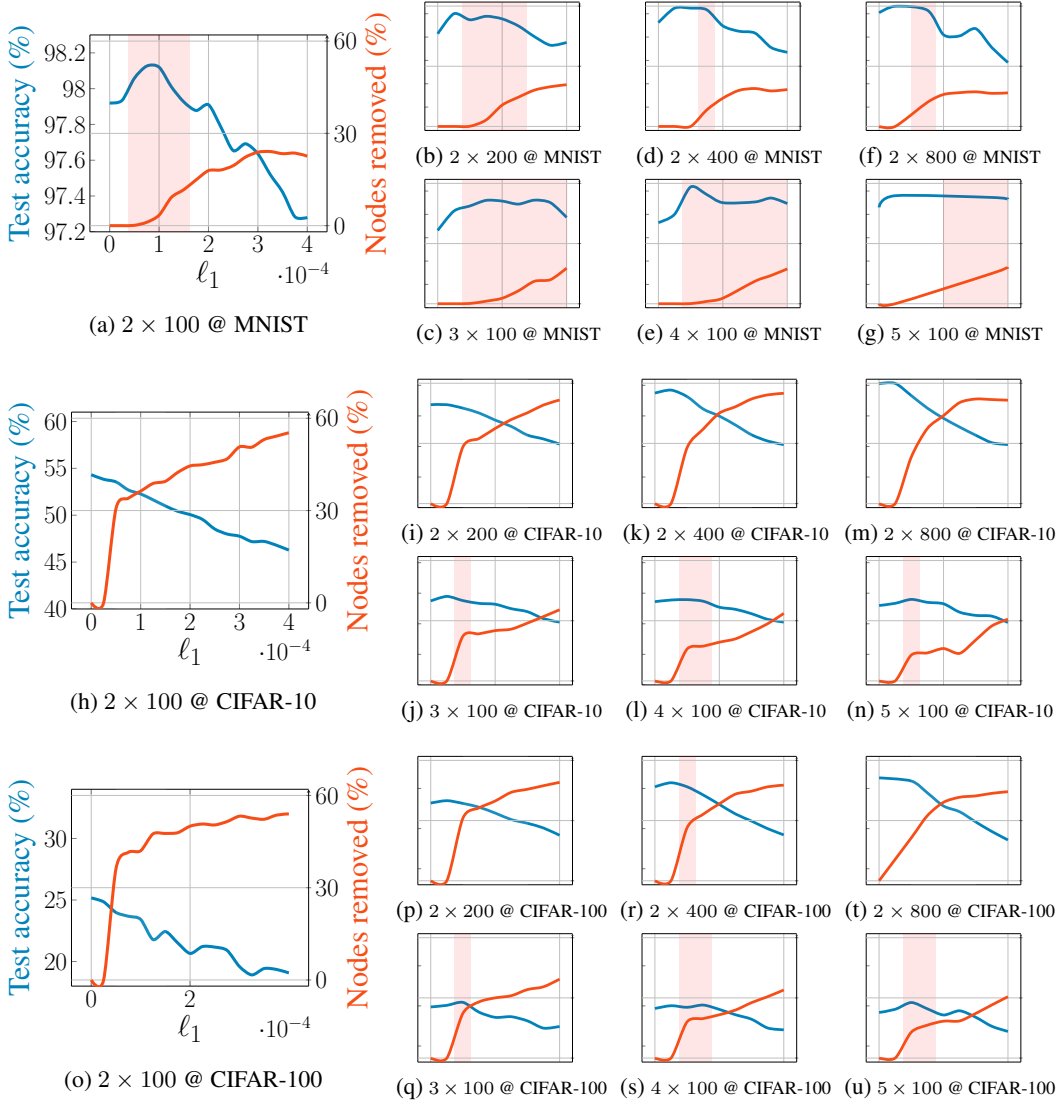

    \begin{minipage}{.37\linewidth}
        \accuracyCompressionPlotWithLabels{0.46}{97.2}{98.3}{2.5pt}{data/mnist_2x100.txt}{0, 0.0001, 0.0002, 0.0003, 0.0004}{red}{37.5}{162.5}
        \vspace{-0.3cm}
        \subcaption{$2 \times 100$ @ MNIST}
    \end{minipage}
    \hfill
    \begin{minipage}{.2\linewidth}
        \begin{subfigure}{\linewidth}
            \accuracyCompressionPlot{0.3}{97.2}{98.3}{3.5pt}{data/mnist_2x200.txt}{0, 0.0001, 0.0002, 0.0003, 0.0004}{red}{37.5}{137.5}
            \vspace{-0.1cm}
            \subcaption{\scriptsize $2 \times 200$ @ MNIST}
            \vspace{0.1cm}
        \end{subfigure}
        \begin{subfigure}{\linewidth}
            \accuracyCompressionPlot{0.3}{97.2}{98.3}{3.5pt}{data/mnist_3x100.txt}{0, 0.0001, 0.0002, 0.0003, 0.0004}{red}{37.5}{200}
            \vspace{-0.1cm}
            \subcaption{\scriptsize $3 \times 100$ @ MNIST}
        \end{subfigure}
    \end{minipage}
    \hfill
    \begin{minipage}{.2\linewidth}
        \begin{subfigure}{\linewidth}           
            \accuracyCompressionPlot{0.3}{97.2}{98.3}{3.5pt}{data/mnist_2x400.txt}{0, 0.0001, 0.0002, 0.0003, 0.0004}{red}{62.5}{87.5}
            \vspace{-0.1cm}
            \subcaption{\scriptsize $2 \times 400$ @ MNIST}
            \vspace{0.1cm}
        \end{subfigure}
        \begin{subfigure}{\linewidth}            
            \accuracyCompressionPlot{0.3}{97.2}{98.3}{3.5pt}{data/mnist_4x100.txt}{0, 0.0001, 0.0002, 0.0003, 0.0004}{red}{37.5}{200}
            \vspace{-0.1cm}
            \subcaption{\scriptsize $4 \times 100$ @ MNIST}
        \end{subfigure}
    \end{minipage}
    \hfill
    \begin{minipage}{.2\linewidth}
        \begin{subfigure}{\linewidth}
            \accuracyCompressionPlot{0.3}{97.2}{98.3}{3.5pt}{data/mnist_2x800.txt}{0, 0.0001, 0.0002, 0.0003, 0.0004}{red}{50}{87.5}
            \vspace{-0.1cm}
            \subcaption{\scriptsize $2 \times 800$ @ MNIST}
            \vspace{0.1cm}
        \end{subfigure}
        \begin{subfigure}{\linewidth}
            \accuracyCompressionPlot{0.3}{97.2}{98.3}{3.5pt}{data/mnist_5x100.txt}{0, 0.0001, 0.0002, 0.0003, 0.0004}{red}{100}{200}
            \vspace{-0.1cm}
            \subcaption{\scriptsize $5 \times 100$ @ MNIST}
        \end{subfigure}
    \end{minipage}

    \vspace{0.4cm}
    
    \begin{minipage}{.36\linewidth}
        \accuracyCompressionPlotWithLabels{0.46}{40}{61}{2.5pt}{data/cifar/cifar_2x100.txt}{0, 0.0001, 0.0002, 0.0003, 0.0004}{white}{50}{150}
        \vspace{-0.3cm}
        \subcaption{$2 \times 100$ @ CIFAR-10}
    \end{minipage}
    \hfill
    \begin{minipage}{.2\linewidth}
        \begin{subfigure}{\linewidth}
            \accuracyCompressionPlot{0.3}{40}{61}{3.5pt}{data/cifar/cifar_2x200.txt}{0, 0.0001, 0.0002, 0.0003, 0.0004}{white}{50}{150}
            \vspace{-0.1cm}
            \subcaption{\scriptsize $2 \times 200$ @ CIFAR-10}
            \vspace{0.1cm}
        \end{subfigure}
        \begin{subfigure}{\linewidth}
            \accuracyCompressionPlot{0.3}{40}{61}{3.5pt}{data/cifar/cifar_3x100.txt}{0, 0.0001, 0.0002, 0.0003, 0.0004}{red}{37.5}{62.5}
            \vspace{-0.1cm}
            \subcaption{\scriptsize $3 \times 100$ @ CIFAR-10}
        \end{subfigure}
    \end{minipage}
    \hfill
    \begin{minipage}{.2\linewidth}
        \begin{subfigure}{\linewidth}           
            \accuracyCompressionPlot{0.3}{40}{61}{3.5pt}{data/cifar/cifar_2x400.txt}{0, 0.0001, 0.0002, 0.0003, 0.0004}{white}{50}{150}
            \vspace{-0.1cm}
            \subcaption{\scriptsize $2 \times 400$ @ CIFAR-10}
            \vspace{0.1cm}
        \end{subfigure}
        \begin{subfigure}{\linewidth}            
            \accuracyCompressionPlot{0.3}{40}{61}{3.5pt}{data/cifar/cifar_4x100.txt}{0, 0.0001, 0.0002, 0.0003, 0.0004}{red}{37.5}{87.5}
            \vspace{-0.1cm}
            \subcaption{\scriptsize $4 \times 100$ @ CIFAR-10}
        \end{subfigure}
    \end{minipage}
    \hfill
    \begin{minipage}{.2\linewidth}
        \begin{subfigure}{\linewidth}
            \accuracyCompressionPlot{0.3}{40}{61}{3.5pt}{data/cifar/cifar_2x800.txt}{0, 0.0001, 0.0002, 0.0003, 0.0004}{white}{50}{150}
            \vspace{-0.1cm}
            \subcaption{\scriptsize $2 \times 800$ @ CIFAR-10}
            \vspace{0.1cm}
        \end{subfigure}
        \begin{subfigure}{\linewidth}
            \accuracyCompressionPlot{0.3}{40}{61}{3.5pt}{data/cifar/cifar_5x100.txt}{0, 0.0001, 0.0002, 0.0003, 0.0004}{red}{37.5}{62.5}
            \vspace{-0.1cm}
            \subcaption{\scriptsize $5 \times 100$ @ CIFAR-10}
        \end{subfigure}
    \end{minipage}
    
    \vspace{0.4cm}
    
    \begin{minipage}{.36\linewidth}
        \accuracyCompressionPlotWithLabels{0.46}{18}{34}{2.5pt}{data/cifar_hundred/cifar_hundred_2x100.txt}{0,0.0002}{white}{50}{150}
        \vspace{-0.3cm}
        \subcaption{$2 \times 100$ @ CIFAR-100}
    \end{minipage}
    \hfill
    \begin{minipage}{.2\linewidth}
        \begin{subfigure}{\linewidth}
            \accuracyCompressionPlot{0.3}{18}{34}{3.5pt}{data/cifar_hundred/cifar_hundred_2x200.txt}{0,0.0002}{white}{50}{150}
            \vspace{-0.1cm}
            \subcaption{\scriptsize $2 \times 200$ @ CIFAR-100}
            \vspace{0.1cm}
        \end{subfigure}
        \begin{subfigure}{\linewidth}
            \accuracyCompressionPlot{0.3}{18}{34}{3.5pt}{data/cifar_hundred/cifar_hundred_3x100.txt}{0,0.0002}{red}{37.5}{62.5}
            \vspace{-0.1cm}
            \subcaption{\scriptsize $3 \times 100$ @ CIFAR-100}
        \end{subfigure}
    \end{minipage}
    \hfill
    \begin{minipage}{.2\linewidth}
        \begin{subfigure}{\linewidth}           
            \accuracyCompressionPlot{0.3}{18}{34}{3.5pt}{data/cifar_hundred/cifar_hundred_2x400.txt}{0,0.0002}{red}{37.5}{62.5}
            \vspace{-0.1cm}
            \subcaption{\scriptsize $2 \times 400$ @ CIFAR-100}
            \vspace{0.1cm}
        \end{subfigure}
        \begin{subfigure}{\linewidth}            
            \accuracyCompressionPlot{0.3}{18}{34}{3.5pt}{data/cifar_hundred/cifar_hundred_4x100.txt}{0,0.0002}{red}{37.5}{87.5}
            \vspace{-0.1cm}
            \subcaption{\scriptsize $4 \times 100$ @ CIFAR-100}
        \end{subfigure}
    \end{minipage}
    \hfill
    \begin{minipage}{.2\linewidth}
        \begin{subfigure}{\linewidth}
            \accuracyCompressionPlot{0.3}{18}{34}{3.5pt}{data/cifar_hundred/cifar_hundred_2x800.txt}{0,0.0002}{white}{50}{150}
            \vspace{-0.1cm}
            \subcaption{\scriptsize $2 \times 800$ @ CIFAR-100}
            \vspace{0.1cm}
        \end{subfigure}
        \begin{subfigure}{\linewidth}
            \accuracyCompressionPlot{0.3}{18}{34}{3.5pt}{data/cifar_hundred/cifar_hundred_5x100.txt}{0,0.0002}{red}{37.5}{87.5}            
            \vspace{-0.1cm}
            \subcaption{\scriptsize  $5 \times 100$ @ CIFAR-100}
        \end{subfigure}
    \end{minipage}

    \caption{
    \textbf{{\color{bmblue}Test accuracy} and {\color{borange}nodes removed} for varying amounts of $\ell_1$ regularization.} The plots correspond to classifiers with different architectures on the (a)-(g) MNIST, 
    (h)-(n) CIFAR-10, and (o)-(u) CIFAR-100 datasets.  
    For each dataset, we keep the ranges of all the axes of the smaller plots  same as the bigger plot but hide the ticks for brevity.
    Networks trained with $\ell_1$ regularization can be exactly compressed, even when regularization improves accuracy. In light red background, test accuracy is better than with no regularization (blue curve) and exact compression occurs (red curve). 
    }
    \label{fig:curves}
\end{figure}

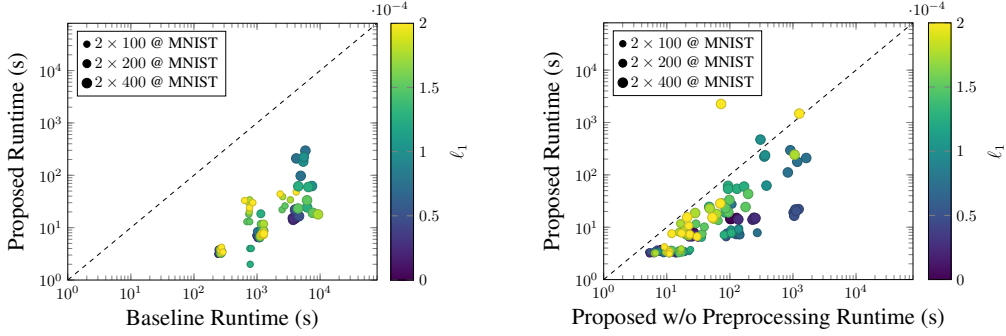
\begin{figure}[!tb]
    \centering
    \begin{minipage}{.49\linewidth}
    \begin{tikzpicture}[scale=0.6]
        \begin{loglogaxis}[
            legend pos=north west,
            colorbar,
            point meta=explicit,
            point meta min=0,
            point meta max=0.0002,
            cycle list name=mark list,
            colormap name=viridis,
            xmin=1,
            xmax=80000,
            ymin=1,
            ymax=80000,
            xlabel={\Large{Baseline Runtime (s)}},
            ylabel={\Large{Proposed Runtime (s)}},
            colorbar style={ylabel={\large$\lOne$}},
        ]
        \draw[dashed]
            (axis cs:1,1) -- 
            (axis cs:80000,80000);
        \addplot+[
            only marks,
            scatter,
            mark=*,
            mark size=2pt]
        table[meta=reg,x=neuron,y=network,discard if not={width}{100}]
        {data/mnist_runtime_compare.txt};
        \addlegendentry{$2 \times 100$ @ MNIST} 
        \addplot+[
            only marks,
            scatter,
            mark=*,
            mark size=2.5pt]
        table[meta=reg,x=neuron,y=network,discard if not={width}{200}]
        {data/mnist_runtime_compare.txt};
        \addlegendentry{$2 \times 200$ @ MNIST}
        \addplot+[
            only marks,
            scatter,
            mark=*,
            mark size=3pt]
        table[meta=reg,x=neuron,y=network,discard if not={width}{400}]
        {data/mnist_runtime_compare.txt};
        \addlegendentry{$2 \times 400$ @ MNIST}
        \end{loglogaxis}
        \end{tikzpicture}
    \end{minipage}
    \hfill
    \centering
    \begin{minipage}{.49\linewidth}
    \begin{tikzpicture}[scale=0.6]
        \begin{loglogaxis}[
            legend pos=north west,
            colorbar,
            point meta=explicit,
            point meta min=0,
            point meta max=0.0002,
            cycle list name=mark list,
            colormap name=viridis,
            xmin=1,
            xmax=80000,
            ymin=1,
            ymax=80000,
            xlabel={\Large{Proposed w/o Preprocessing Runtime (s)}},
            ylabel={\Large{Proposed Runtime (s)}},
            colorbar style={ylabel={\large$\lOne$}, at={(1.05,1)},anchor=north west},
        ]
        \draw[dashed]
            (axis cs:1,1) -- 
            (axis cs:80000,80000);
        \addplot+[
            only marks,
            scatter,
            mark=*,
            mark size=2pt]
        table[meta=reg,x=neuron,y=network,discard if not={width}{100}]
        {data/scatter_1_reg_fig2.txt};
        \addlegendentry{$2 \times 100$ @ MNIST} 
        \addplot+[
            only marks,
            scatter,
            mark=*,
            mark size=2.5pt]
        table[meta=reg,x=neuron,y=network,discard if not={width}{200}]
        {data/scatter_1_reg_fig2.txt};
        \addlegendentry{$2 \times 200$ @ MNIST}
        \addplot+[
            only marks,
            scatter,
            mark=*,
            mark size=3pt]
        table[meta=reg,x=neuron,y=network,discard if not={width}{400}]
        {data/scatter_1_reg_fig2.txt};
        \addlegendentry{$2 \times 400$ @ MNIST}
        \end{loglogaxis}
        \end{tikzpicture}
    \end{minipage}
   \caption{
   \textbf{Comparison of runtimes (in seconds) to identify all stable neurons.} On the left, we compare the proposed approach against the baseline from~\cite{serra2020lossless}. On the right, we compare the proposed approach with preprocessing against the proposed approach without preprocessing. 
   }
    \label{fig:runtimes}
\end{figure}

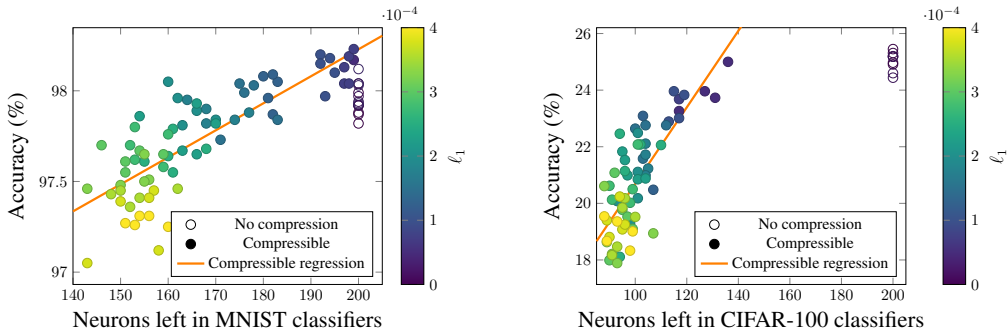
\begin{figure}[!htb]
    \centering
    \begin{minipage}{.49\linewidth}
        \begin{tikzpicture}[scale=0.6]
        \begin{axis}[
            legend pos=south east,
            colorbar,
            point meta=explicit,
            point meta min=0,
            point meta max=0.0004,
            cycle list name=mark list,
            colormap name=viridis,
            xlabel={\Large{Neurons left in MNIST classifiers}},
            ylabel={\Large{Accuracy (\%)}},
            colorbar style={ylabel={\large$\lOne$}},
            xmin=140,
            xmax=205,
        ]
        \addlegendentry{No compression}
        \addplot+[
            only marks,
            scatter,
            mark=o,
            mark size=3pt]
        table[meta=reg,x=nodes,y=accuracy]
        {data/mnist_uncompressed_regression.txt};
        \addplot+[
            only marks,
            scatter,
            mark=*,
            mark size=3pt]
        table[meta=reg,x=nodes,y=accuracy]
        {data/mnist_compressed_regression.txt};
        \addlegendentry{Compressible}
        \draw[color=orange,line width=0.5mm]
            (axis cs:140,97.335) -- 
            (axis cs:205,98.3035);
        \addlegendimage{line width=0.5mm,color=orange}
        \addlegendentry{Compressible regression}
        \end{axis}
        \end{tikzpicture}
    \end{minipage}
    \hfill
    \centering
    \begin{minipage}{.49\linewidth}
        \begin{tikzpicture}[scale=0.6]
        \begin{axis}[
            legend pos=south east,
            colorbar,
            point meta=explicit,
            point meta min=0,
            point meta max=0.0004,
            cycle list name=mark list,
            colormap name=viridis,
            xlabel={\Large{Neurons left in CIFAR-100 classifiers}},
            ylabel={\Large{Accuracy (\%)}},
            colorbar style={ylabel={\large$\lOne$}},
            xmin=85,
            xmax=205,
        ]
        \addlegendentry{No compression}
        \addplot+[
            only marks,
            scatter,
            mark=o,
            mark size=3pt]
        table[meta=reg,x=nodes,y=accuracy]
        {data/cifar100_uncompressed_regression.txt};
        \addplot+[
            only marks,
            scatter,
            mark=*,
            mark size=3pt]
        table[meta=reg,x=nodes,y=accuracy]
        {data/cifar100_compressed_regression.txt};
        \addlegendentry{Compressible}
        \draw[color=orange,line width=0.5mm]
            (axis cs:85,18.6638) -- 
            (axis cs:205,34.8758);
        \addlegendimage{line width=0.5mm,color=orange}
        \addlegendentry{Compressible regression}
        \end{axis}
        \end{tikzpicture}
    \end{minipage}
    \caption{\textbf{Relationship between size of compressed neural network and accuracy on $2 \times 100$ classifiers.} The coefficient of determination (R$^2$) for the linear regression obtained for accuracy based on neurons left for compressible networks is $69$\% on MNIST and $61$\% on CIFAR-100.}
    \label{fig:accuracy}
\end{figure}

\textbf{Runtime improvement ~~}
\myReferFigure{fig:runtimes} compares 
the baseline~\cite{serra2020lossless} with our approach on smaller MNIST classifiers---$2 \times 100$, $2 \times 200$, and $2 \times 400$---using $\lOne$ as described above. The median ratio between runtimes is $100$. The overall speedup is greater in larger networks: the median runtime ratio is $77$ for $2 \times 100$, $153$ for $2 \times 200$, and $193$ for $2 \times 400$. By comparing the runtimes when not timing out with and without the preprocessing step in $2 \times 100$, we observe a median speed up of $3.2$.

\textbf{Effect of regularization on compressibility ~~}
We observe more compression with more $\lOne$ regularization.  
For sufficiently large networks having the same accuracy as those trained with $\lOne = 0$ on MNIST, we can remove around $20$\% of the neurons and $40$\% of the connections. 
In line with~\cite{serra2020lossless}, 
we observe that the exact compressibility of neural networks trained with $\lOne = 0$ is negligible, 
but also that you can have the cake and eat it too: 
certain choices of regularization lead to better accuracy and a smaller network. 
However,the cake can get very expensive as runtimes increase considerably.

\textbf{Relationship between compressibility and accuracy ~~}
\myReferFigure{fig:accuracy} analyzes the relationship between classifier accuracy and the number of neurons left after compression for  $2 \times 100$ classifiers. 
When excluding uncompressible networks with $\lOne = 0$ or sufficiently small, we obtain a linear regressions with coefficient of determination ($R^2$) of $69$\% on MNIST, $91$\% on CIFAR-10, and $61$\% on CIFAR-100. 
That suggests that accuracy is a good proxy for how much a neural network can be compressed.


\textbf{Motivation for exact compression ~~}
We also compared exact compression with Magnitude-based Pruning (MP), one of the most commonly used inexact methods.
First, we identified all the connections that would be pruned by the removal of stably inactive neurons with our approach, which would also be harmless if identified and removed by MP. Second, we ranked all the connections based on the absolute value of their coefficients in order to identify at what pruning ratio those connections would have been removed by MP. We consistently found out across architectures of different sizes and levels of regularization that some of the pruned connections by our method would be found by MP at the $99^\text{th}$ percentile. 
In other words, even though such connections would have no impact if removed from the network, MP would only resort to removing them at very extreme levels of pruning.
Furthermore, if the same pruning ratio is used with MP, on average $10$\% of the total number of connections---or $18$\% of the pruned connections---removed by our method would not be removed by MP.

\subsection{Limitations and alternatives}
Due to the use of MILP solvers, our approach is not applicable to very large networks. In what follows, we consider ways to extend our approach by lifting some or all the guarantees provided. 

\textbf{Inexact approach to larger networks ~~}
We evaluated the impact of using only the quick preprocessing step described in Section~\ref{sub:preprocess} to determine which neurons to remove. Note that the preprocessing step is in principle intended to identify neurons that are not stable in order to avoid spending further time on them, but we can conversely assume that all the other neurons are stable at the cost of removing more than we should.
By using preprocessing alone, we identified on average $31.93$\% potentially stable neurons in MNIST classifiers, $40.86$\% in CIFAR-10 classifiers, and $41.98$\% in CIFAR-100 classifiers. Among those neurons, only a few were actually not stable when evaluated with the test set. In terms of the number of not stable neurons with respect to the test set divided by the number of stable neurons with respect to the training set, we would have removed $1.16$\% more neurons that we should for MNIST, $0.60$\% for CIFAR-10, and $1.19$\% for CIFAR-100. We refer to Appendix~\ref{ap:cnn} for experiments involving convolutional neural networks~(CNNs).

\textbf{Restricting exact compression to more likely inputs ~~}
We also tested the effect of bounding the sum of all the MNIST inputs to be within the minimum and maximum observed values. In particular, we have constrained the sum of all inputs to be within the interval $[15, 320]$ instead of $[0, 784]$. We believed that this approach would be preferable to constraining the value of individual inputs, since that would have affected the output upon rotation and translation.
Note that this constraint is equivalent to imposing a prior on the number of foreground pixels on the digits to be within a range. Global priors that jointly act on all the pixels have been used in computer vision in the pre-deep learning era, e.g., \cite{delong2012fast}.
By restricting the analysis to the cases in which the time limit has not been exceeded either before or after the change, we obtained a better runtime in $69.6$\% of the cases and the runtime geometric mean went down by $17.7$\%.

\section{Conclusion}\label{sec:conclusion}

This paper outlined the potential for exact compression of neural networks and presented an approach that makes it practical for sizes that are large enough for many applications.  
To the best of our knowledge, 
our approach is the state-of-the-art for optimization-based exact compression.
Our performance improvements come from insights about the MILP formulations associated with optimization problems over neural networks, 
which have many other applications besides exact compression. Most notably, such formulations are also used for network verification~\cite{Bunel2017,Changliu2019,rossig2020verification}.

\textbf{Societal Impact ~~}
Large models are resource-intensive for both training as well as inference. In contrast to approximate methods, our exact model compression algorithms can help deep learning practitioners to save computational time and resources without worrying about any loss in performance. That helps preventing the documented side effect of disproportionally degrading performance for some classes more than for other classes when the indicator of a successful compression is the overall performance, which could also lead to fairness issues~\cite{hooker2019forget,paganini2020forget,hooker2020bias}.

\paragraph{Acknowledgements ~~}
Thiago Serra was supported by the National Science Foundation (NSF) grant IIS 2104583.
Xin Yu and Srikumar Ramalingam were partially funded by the NSF grant IIS 1764071.
Abhinav Kumar was partially funded by the Ford Motor Company and the Army Research Office (ARO) grant W911NF-18-1-0330.
We also thank the anonymous reviewers for their constructive feedback that helped in shaping the final manuscript.

{
\small
\bibliographystyle{plain}
\bibliography{references}
}

\clearpage
\appendix 
\begin{center}
    \textbf{\Large Scaling Up Exact Neural Network Compression by ReLU Stability\\[12pt] Supplementary Material\\[18pt]}
\end{center}

\renewcommand{\thesection}{A\arabic{section}}

\section{Description of MILP formulation for a ReLU activation}\label{ap:mip_relu}

The formulation below is used to identify inputs for which a given output or activation pattern can be achieved. 
The decision variables include 
(i) the vector $\vx^0$ associated with the input of the neural network; 
(ii) the vector $\vy^l$ associated with the preactivation output of each hidden layer of the neural network; 
(iii) the vector $\vx^l$ associated with the output of each hidden layer of the neural network; 
(iv) the vector $\chi^l$ associated with the complementary output of each hidden layer of the neural network; and 
(v) the binary vector $\vz^l$ defining which neurons are active or not in each hidden layer of the neural network.
The vector of weights $\vw^l_i$ and the bias $b^l_i$ associated with each neuron as well as the constants $M^l_i$ and $\mu^l_i$ are coefficients of the formulation.
The constraints are as follows:
\begin{align}
    \vw^l_i \cdot \vx^{l-1} + b^l_i = y^l_i \label{cons:y_in} \\
    y^l_i = x^l_i - \chi^l_i \label{cons:y_out} \\
    x^l_i \leq M^l_i z^l_i \label{cons:x_ub} \\
    \chi^l_i \leq \mu^l_i (1 - z^l_i) \label{cons:chi_ub} \\
    x^l_i \geq 0 \label{cons:x_lb} \\
    \chi^l_i \geq 0 \label{cons:chi_lb} \\
    z^l_i \in \{0, 1\} \label{cons:z_bin}
\end{align}
Constraint~\eqref{cons:y_in} matches the layer input $\vx^{l-1}$ with the neuron preactivation output $y^l_i$. 
We then use the binary variable $z^l_i$ to match $y^l_i$ with the neuron output with either $x^l_i$ or 0. 
When $z^l_i = 1$, constraints~\eqref{cons:chi_ub} and \eqref{cons:chi_lb} imply that $\chi^l_i = 0$, 
and thus $x^l_i = y^l_i$ due to constraint~\eqref{cons:y_out}. 
That only happens if $y^l_i \geq 0$ due to constraint~\eqref{cons:x_lb}. 
When $z^l_i = 0$, constraints~\eqref{cons:x_ub} and \eqref{cons:x_lb} imply that $x^l_i=0$, 
and thus $\chi^l_i = -y^l_i$. 
That only happens if $y^l_i \leq 0$ due to constraint~\eqref{cons:chi_lb}.

\section{On dropping constraint~\eqref{cons:pq_bin}}\label{ap:binary_relaxation}
We avoid explicitly enforcing that  variables $p^l_i$ and $q^l_i$ are binary by leveraging that $z^l_i$ is binary.
Constraint~\eqref{cons:p} implies that $p^l_i \in [0,1]$ and $p^l_i \neq 0$ 
only if $z^l_i = 1$.
In turn, if $z^l_i=1$, 
then we can assume $p^l_i = 1$ by optimality since the objective function~\eqref{cons:of} maximizes the sum of those variables and no other constraint limits its value. 
Likewise, 
constraint~\eqref{cons:q} implies that $q^l_i \in [0,1]$ and $q^l_i \neq 0$ 
only if $z^l_i = 0$.
In turn, if $z^l_i=0$, 
then likewise we can assume $q^l_i = 1$ by optimality since the objective function~\eqref{cons:of} maximizes the sum of those variables and no other constraint limits its value.
Reducing the number of binary variables is widely regarded as a good practice to make MILP formulations easier to solve.

\section{Proofs from Section~\ref{sub:milp}}\label{ap:milp_proofs}

\newtheorem*{prop:cpq0}{Proposition \ref{prop:cpq0}}
\begin{prop:cpq0}
If $\mathcal{C}(\vP,\vQ)=0$, 
then every neuron $i \in \sP^l$ is stably inactive 
and every neuron $i \in \sQ^l$ is stably active. 
\end{prop:cpq0}
\begin{proof}
Constraint~\eqref{cons:p} is the only upper bound on $p^l_i$ besides constraint~\eqref{cons:pq_bin}. 
Hence, 
if there is any solution $(\bar{x},\bar{z},\bar{p},\bar{q})$ of~\eqref{cons:pq_unit}--\eqref{cons:pq_bin} in which $\bar{z}^l_i = 1$ for some $i \in \sP^l_i, l \in \sL$, 
then either $\bar{p}^l_i = 1$ or there is another solution $(\bar{\bar{x}},\bar{\bar{z}},\bar{\bar{p}},\bar{\bar{q}})$ in which $\bar{\bar{p}}^l_i = 1$ and all other variables have the same value. 

Likewise, 
constraint~\eqref{cons:p} is the only upper bound on $q^l_i$ besides constraint~\eqref{cons:pq_bin}. 
Hence, 
if there is any solution $(\bar{x},\bar{z},\bar{p},\bar{q})$ of~\eqref{cons:pq_unit}--\eqref{cons:pq_bin} in which $\bar{z}^l_i = 0$ for some $i \in \sP^l_i, l \in \sL$, 
then either $\bar{q}^l_i = 1$ or there is another solution $(\bar{\bar{x}},\bar{\bar{z}},\bar{\bar{p}},\bar{\bar{q}})$ in which $\bar{\bar{q}}^l_i = 1$ and all other variables have the same value.

If $\mathcal{C}(\vP,\vQ)=0$, 
then for every solution $(\bar{x},\bar{z},\bar{p},\bar{q})$ it follows that $\bar{p}^l_i = 0~\forall i \in \sP^l, l \in \sL$ and $\bar{q}^l_i = 0~\forall i \in \sQ^l, l \in \sL$, 
and consequently $\bar{z}^l_i = 0~\forall i \in \sP^l, l \in \sL$ and $\bar{z}^l_i = 1~\forall i \in \sQ^l, l \in \sL$. 
Thus,  
the neurons in $\sP^l$ are always inactive and the neurons in $\sQ^l$ are always active for any valid input.
\end{proof}

\newtheorem*{cor:np1}{Corollary \ref{cor:np1}}
\begin{cor:np1}
The stability of all neurons of a neural network can be determined by solving formulation~\eqref{cons:of}--\eqref{cons:pq_bin} at most $N+1$ times, 
where $N := \sum\limits_{l \in \sL} n_l$.
\end{cor:np1}
\begin{proof}
Let us initially consider a formulation in which $\sP^l = \sQ^l = \{1, \ldots, n_l\}~\forall l \in \sL$ and then respectively remove from those sets each neuron $i$ for which $p^l_i = 1$ and $q^l_i = 1$ in any solution obtained. 
When the formulation is first solved, 
we remove each neuron from either $\sP^l$ or $\sQ^l$, 
and therefore $N$ states remain unobserved. 
In subsequent steps, 
either (i) $\mathcal{C}(\vP,\vQ)>0$ and the number of unobserved states decreases; or (ii) $\mathcal{C}(\vP,\vQ)=0$, and thus any neuron $i \in \sP^l$ is stably inactive and any neuron $i \in \sQ^l$ is stably active.
\end{proof}

\section{On lazy constraint callbacks}\label{ap:callback}

Lazy constraint callbacks are generally used when the total number of constraints of an MILP formulation is prohibitively large. 
One such example is the most commonly used formulation for the traveling salesperson problem due to the subtour elimination constraints~\cite{dantzig1954tsp}. 
The callback allows us to handle such cases more efficiently by formulating the problem with fewer constraints and then adding the remaining ones only if they are necessary to rule out infeasible solutions. 
Every time that a supposedly feasible solution is found, 
the MILP solver invokes the callback implemented by the user 
for an opportunity to make such a solution infeasible by adding one of the missing constraints that the supposedly feasible solution does not satisfy. 
If none is provided by the callback, the MILP solver accepts the solution as feasible. 

In our case, 
we use a lazy constraint callback for a slightly different purpose. 
Namely, 
we implement the callback to (i) inspect every feasible solution that is obtained; and (ii) mimic the updates that would have been made to $\vP$ and $\vQ$ between consecutive calls to the solver by adding constraints that set the value of either $p^l_i$ or $q^l_i$ to zero once a solution is found in which such variable has a positive value. 
In other words, 
the callback adds constraints to ignore the effect of $p^l_i$ or $q^l_i$ on the objective function if we know that the $i$-th neuron of layer $l$ is active or inactive for some input, respectively. 
Therefore, the MILP solver will eventually produce an optimal solution of value zero once the set of solutions inspected by the callback covers all the possible states for the neurons and the remaining states are deemed unattainable after an exhaustive search.


\section{A revised algorithm for compressing the neural network}\label{ap:compression_algo}

\begin{algorithm}[h!]
  \caption{\texttt{LEO++} performs exact compression of a neural network with a single operation per layer}
  \label{alg:new_leo}
\begin{algorithmic}[1]
  \State {\bfseries Input:} neural network $\left(L, \left\{ (n_l, \mW^l, \vb^l) \right\}_{l \in \sL}\right)$ and stable neurons  $\left(\left\{ (\sP^l, \sQ^l) \right\}_{l \in \sL}\right)$
  \For{$l \gets 1$ {\bfseries to} $L$}
  \Comment{Loops over all hidden layers}
  \label{alg:b_compression}
     \If{$|\sP^l| = n_l$} \label{alg:b_collapse}
      \Comment{Entire layer is stably inactive}
      \State find output $\overline{\vx}^L$ for an arbitrary input $\overline{\vx}^0 \in \sX$
      \State remove all layers except $L$, which becomes 1 
      \State $\mW^1 \gets \bm{0}$ ~~~and~~~ $\vb^L \gets \overline{\vx}^L$
      \State {\bfseries break}
      \Comment{All hidden layers were collapsed}
      \label{alg:e_collapse}
     \ElsIf{$|\sP^l|+|\sQ^l| = n_l$ and $l<L$}
     \Comment{Entire layer is stable, but not inactive}
     \label{alg:b_fold}
      \State $\mW^{l+1} \gets \mW^{l+1} \mI_{n_l}(\sQ^l) \mW^{l}$ ~~~and~~~ $\vb^{l+1} \gets \mW^{l+1} \mI_{n_l}(\sQ^l) \vb^l + \vb^{l+1}$
      \State remove layer $l$ 
      \Comment{Hidden layer was folded}
      \label{alg:e_fold}
     \ElsIf{$l < L$}
      \State $r \gets \text{rank}\left(\mW^l_{\sQ^l}\right)$ \label{alg:b_stably_active} 
      \If{$r < |\sQ^l|$ and $l<L$}
         \State find $\overline{\sQ} \subset \sQ^l$ such that $r = |\overline{\sQ}| = \text{rank}\left(\mW^l_{\overline{\sQ}}\right)$
         \For{{\bfseries every} $i \in \sQ^l \setminus \overline{\sQ}$}
          \State find $\{ \alpha^i_j \}_{j \in \overline{\sQ}}$ such that $\vw^l_i = \sum_{j \in \overline{\sQ}} \alpha^i_j \vw^l_j$
         \EndFor
         \For{$k \gets 1$ {\bfseries to} $n_{l+1}$}
          \For{{\bfseries every} $j \in \overline{\sQ}$}
             \State $w^{l+1}_{kj} \gets w^{l+1}_{kj} + \sum_{i \in \sQ^l \setminus \overline{\sQ}} \alpha^i_j w^{l+1}_{ki}$
          \EndFor
          \State $b^{l+1}_k \gets b^{l+1}_k + \sum_{i \in \sQ^l \setminus \overline{\sQ}} w^{l+1}_{ki} \left( b^l_i - \sum_{j \in \overline{\sQ}} \alpha^i_j b^l_j \right)$
         \EndFor
         \State remove from layer $l$ every neuron $i \in \sQ^l \setminus \overline{\sQ}$ \label{alg:e_stably_active}
         \State remove from layer $l$ every neuron $i \in \sP^l$ \label{alg:stably_inactive}
      \EndIf
     \EndIf
  \EndFor \label{alg:e_compression}
\end{algorithmic}
\end{algorithm}

Algorithm~\ref{alg:new_leo}, which we denote as \texttt{LEO++} (Lossless Expressiveness Optimization, as in \cite{serra2020lossless}), leverages neuron stability for exactly compressing neural networks. 
We describe next each form of compression contained in the algorithm. 
For ease of explanation, 
they are in reverse order of appearance. 
These compression operations are the same as in \cite{serra2020lossless}, but performed once per layer instead of once per neuron. 
In comparison to the original algorithm $\texttt{LEO}$, 
the order of the operations is such that (i) neurons are not removed or merged if the entire layer is going to be folded; and (ii) special cases such as a neuron with weight vector $\vw^l_i = \bm{0}$ do not need to be considered apart. 
For the most elaborate operations, we prove their correctness when applied to the entire layer.  

\textbf{Removal of stably inactive neurons ~~}
This operation is performed in line~\ref{alg:stably_inactive}.
Since the output of stably inactive neurons is always 0, we remove those neurons without affecting subsequent computations. 
The case in which an entire layer is stably inactive is considered separately. 

\textbf{Merging of stably active neurons ~~}
This operation is performed between lines~\ref{alg:b_stably_active} and \ref{alg:e_stably_active}. We use the following results to show how stably active neurons can be merged.

\begin{proposition}\label{prop:ld}
Let $\sS$ be a set of stably active neurons in layer $l$. If $r := \text{rank}(\mW^l_S) < |S|$ and let $\sT \subset \sS$ be a subset of those neurons for which $\text{rank}(\mW^l_{\sT}) = r$, then the output of the neurons in $\sS \setminus \sT$ is an affine function on the output of the neurons in $\sT$.
\end{proposition}
\begin{proof}
For every $i \in \sS \setminus \sT$, 
there is a vector $\bm{\alpha}^i \in \mathbb{R}^{r}$ 
such that $\vw^l_i = \sum_{j \in \sT} \alpha^i_j \vw^l_j$. 
Since $\vx^l_i = \vw^l_i \cdot \vx^{l-1} + b^l_i$ for every $i \in \sS$ because all neurons in $\sS$ are stably active, 
then for every $i \in \sS \setminus \sT$ it follows that 
$\vx^l_i = \sum_{j \in \sT} \alpha^i_j \vw^l_j \cdot \vx^{l-1} + b^l_i = \sum_{j \in \sT} \alpha^i_j \left(\vw^l_j \cdot \vx^{l-1} + b^l_j \right) + \left( b^l_i - \sum_{j \in \sT} \alpha^i_j b^l_j \right) = \sum_{j \in \sT} \alpha^i_j x^l_j + \left( b^l_i - \sum_{j \in \sT} \alpha^i_j b^l_j \right)$. 
\end{proof}

\begin{corollary}\label{cor:stable_neurons}
If $\sS$, $\sT$, and $l$ are such as in Proposition~\ref{prop:ld}, 
then the pre-activation output of the neurons in layer $l+1$ is an affine function on the outputs of all neurons from layer $l$ with exception of the neurons in $\sT$.
\end{corollary}
\begin{proof}
Let $\sU := \{1, \ldots, n_l\} \setminus \sS$. 
The pre-activation output of every neuron $i$ in layer $l+1$ 
is given by $y^{l+1}_i = \sum\limits_{j \in \sU \cup \sS} w^{l+1}_{ij} x^l_j + b^{l+1}_i = \sum\limits_{j \in \sU \cup \sT} w^{l+1}_{ij} x^l_j + \sum\limits_{j \in \sS \setminus \sT} w^{l+1}_{ij} \left(\sum\limits_{k \in \sT} \alpha^j_k x^l_k + \left( b^l_j - \sum\limits_{k \in \sT} \alpha^j_k b^l_k \right) \right) + b^{l+1}_i = 
\sum\limits_{j \in \sU} w^{l+1}_{ij} x^l_j + 
\sum\limits_{j \in \sT} \left( w^{l+1}_{ij} + \sum\limits_{k \in \sS \setminus \sT} \alpha^k_j w^{l+1}_ik \right) x^l_j + \left( b^{l+1}_i + \sum\limits_{j \in \sS \setminus \sT} w^{l+1}_{ij}\left( b^l_j - \sum\limits_{k \in \sT} \alpha^j_k b^l_k \right) \right)$.
\end{proof}

In Algorithm~\ref{alg:new_leo}, we use relationships implied by the proof of Corollary~\ref{cor:stable_neurons} with $\sS = \sQ^l$ and $\sT = \overline{\sQ}$ to merge stably active neurons. By adjusting the biases of the neurons in the next layer as well as the weights connecting every neuron in $\overline{\sQ}$ with the neurons in the next layer, 
we assign a weight of 0 to the connections between every neuron in $\sQ^l \setminus \overline{\sQ}$ and the neurons in the next layer. 
Hence, we simply remove all neurons in $\sQ^l \setminus \overline{\sQ}$ after adjusting those network parameters.

The case in which an entire layer is stably active---either before any compression is applied or once stably inactive neurons are removed---is also considered separately.

\textbf{Folding of stable layers ~~}
This operation is performed between lines~\ref{alg:b_fold} and \ref{alg:e_fold}. 
We use the following results to show that stable layers can be folded in a single step.

\begin{proposition}\label{prop:stable_layer}
If all the neurons of layer $l \in \sL \setminus \{ L \}$ are stably active, then the pre-activation output of layer $l+1$ is an affine function on the inputs of layer $l$.
\end{proposition}
\begin{proof}
Since $\vx^l = \mW^l \vx^{l-1} + \vb^l$, 
then 
$\vy^{l+1} = \mW^{l+1} \vx^l + \vb^{l+1} = \mW^{l+1} \mW^l \vx^{l-1} + \left( \mW^{l+1} \vb^l + \vb^{l+1} \right)$. 
\end{proof}

\begin{corollary}\label{cor:stable_layer}
If all neurons of layer $l \in \sL \setminus \{ L \}$ are stable, then the pre-activation output of layer $l+1$ is an affine function on the inputs of layer $l$.
\end{corollary}

\begin{proof}
Let $\sS$ be the set of stably active neurons in layer $l$. 
If $|\sS| < n_l$, 
the identity $\vx^l = \mW^l \vx^{l-1} + \vb^l$ still holds 
if the bias and the weights of all the connections of the neurons not in $\sS$ with the neurons in the next layer are 0. 
More generally, 
we can thus obtain an equivalent neural network if $\mW^l$ and $\vb^l$ are both premultiplied by $\mI_{n_l}(\sS)$ since that only would change the weights and biases associated with the neurons not in $\sS$ to 0.
Hence, 
the identity $\vx^l = \mI_{n_l}(\sS) \left( \mW^l \vx^{l-1} + \vb^l \right)$ always holds if all neurons in layer $l$ are stable, 
which implies that $\vy^{l+1} = \mW^{l+1} \mI_{n_l}(\sS) \mW^l \vx^{l-1} + \left( \mW^{l+1} \mI_{n_l}(\sS) \vb^l + \vb^{l+1} \right)$.
\end{proof}

In Algorithm~\ref{alg:new_leo}, we use relationships implied by the proof of Corollary~\ref{cor:stable_layer} with $\sS = \sQ^l$ to fold stable layers. 
By adjusting the biases and the weights of layer $l+1$, 
that layer directly uses the outputs from layer $l-1$. 

Although the steps above would apply if a layer is stably inactive, that case deserves separate consideration.

\textbf{Collapse of a network with stably inactive layers ~~}
This operation is performed between lines~\ref{alg:b_collapse} and \ref{alg:e_collapse}. 
If layer $l \in \sL$ are stably inactive, 
then $\vx^l = 0$ for any input $\vx^0 \in \sX$ and thus the value of $\vx^L$ is constant. 
Hence, we collapse layers 1 to $L-1$ by making the output of the remaining layer constant and equal to such value of $\vx^L$.

\subsection{On the complexity of the new algorithm}

While \texttt{LEO++} requires solving fewer optimization problems than \texttt{LEO}~\cite{serra2020lossless}, 
the dependence on solving a single NP-hard problem---such as MILP formulations in general---implies an exponential worst-case complexity. 
Nevertheless, the progress of MILP in the past decades makes it possible to solve considerably large problems with state-of-art MILP solvers. 
In that context, 
the computational experiments
are a more appropriate indicator of performance improvements than complexity considerations.

\clearpage
\section{Implementation details}\label{ap:implementation}

We  now  provide  additional experimental results evaluating  our  proposed method and the baseline.

\textbf{Architecture and Loss~~}
We implemented the fully connected architectures in PyTorch~\cite{paszke2019pytorch}. 
All the networks have ReLU activations but have varying number of layers and width. 
For the classifiers, we pass the output through a softmax layer and use negative log-likelihood loss as the loss function. 
For the autoencoders, we use MSE loss as the loss function.

\textbf{Datasets and Splits~~}
We keep the output units at $10$ and $784$ for the MNIST dataset~\cite{lecun1998mnist} classifiers and autoencoders, respectively.
We keep the output units at $10$ and $100$ for the CIFAR-10 and the CIFAR-100 dataset~\cite{krizhevsky2009learning} classifiers, respectively.
We use the standard train-validation data splits of each of the datasets available in PyTorch.

\textbf{Data Augmentation~~}
We do not do any data augmentation of training images of the MNIST dataset as in~\cite{serra2020lossless} for a fair comparison. We carry out the standard data augmentation of training images of the CIFAR-10 and CIFAR-100 datasets: horizontal flipping with probability $0.5$, random rotation in the range between $(-10^{o}, 10^{o})$, random scaling in the range $(0.8, 1.2)$, random shear parallel to the x axis in the range $(-10, 10)$, and scaling the brightness, contrast, saturation and hue by a random factor in the range $(0.8, 1.2)$. 

\textbf{Optimization~~}
Training proceeds from scratch for $120$ epochs and starts with learning rate of $0.01$, which is decayed by a factor of $0.1$ after every $50$ epochs as in~\cite{serra2020lossless}. 
We use SGD with momentum optimizer, with a momentum of $0.9$ and batch size $128$ as in~\cite{serra2020lossless}. 
Unless stated otherwise, we use $\lOne$ regularization.
We keep the weight decay at $0$ unless otherwise stated. 
We consider the model saved in the last epoch as our final model.

\textbf{MILP Solver~~}
We solve the MILP formulations using Gurobi 9.1.0 through gurobipy~\cite{gurobi2020gurobi}. 
We calculate the value of the positive constants $M^l_i$ and  $\mu^l_i$ for each neuron with an upper bound of on the values of $x^l_i$ and $\chi^l_i$ through interval arithmetic by taking element-wise maxima~\cite{cheng2017mip}.

\textbf{Initialization~~}
We initialize the weights of the network with the Kaiming initialization~\cite{he2016deep} and the biases to zero with different random seeds for each training. 
We train every setting $5$ times, and get the stably active and inactive neurons with the proposed approach to prune the network for each run. We omit from the summaries the runs which resulted in a time out. We keep the timeout to $3$ hours. 

\textbf{Hardware~~}
We ran the classifier experiments on a machine with Intel Core i7-4790 CPU @ $3.60$ GHz processor, $32$ GB of RAM, and one $4$ GB Nvidia GeForce GTX 970 GPU.
The autoencoder experiments were run on a machine with $40$ Intel Xeon E5-2640 CPU @ $2.40$ GHz processors, $126$ GB of RAM, and one $12$ GB Nvidia Titan Xp GPU. 


\clearpage
\section{Additional experiments and results}\label{ap:experiments}

\subsection{MNIST Classifiers}

\textbf{Relationship between Runtime and Regularization ~~}
\myReferTable{tab:mnist_depth} and \myReferTable{tab:mnist_width} show the runtime achieved by the proposed method at different $\lOne$ regularization on MNIST classifiers.

\begin{table}[!htb]
\caption{\textbf{MNIST Classifiers:} Compression results with fixed width and varying depth.}
\label{tab:mnist_depth}
\vskip 0.15in
\begin{center}
\begin{small}
\begin{sc}
\begin{tabular}{ccccccc}
\toprule
&&& Compression & \multicolumn{2}{c}{\% Removed} & Timed \\
Arch. & $\lOne$ & Accuracy (\%) & Runtime (s) & Neurons & Connections & Out \\
\midrule
2 $\times$ 100 & 0 & 97.92 $\pm$ 0.09 & 3.4 $\pm$ 0.3 & 0 $\pm$ 0 & 0 $\pm$ 0 & 0 \\
2 $\times$ 100 & 0.000025 & 97.93 $\pm$ 0.02 & 3.2 $\pm$ 0.1 & 0 $\pm$ 0 & 0 $\pm$ 0 & 0 \\
2 $\times$ 100 & 0.00005 & 98.06 $\pm$ 0.09 & 3.5 $\pm$ 0.3 & 0.1 $\pm$ 0.2 & 0.2 $\pm$ 0.4 & 0 \\
2 $\times$ 100 & 0.000075 & 98.13 $\pm$ 0.09 & 3.2 $\pm$ 0.2 & 1.1 $\pm$ 0.4 & 2 $\pm$ 0.8 & 0 \\
2 $\times$ 100 & 0.0001 & 98.12 $\pm$ 0.09 & 3.5 $\pm$ 0.1 & 3.4 $\pm$ 0.7 & 6 $\pm$ 1 & 0 \\
2 $\times$ 100 & 0.000125 & 98.01 $\pm$ 0.09 & 3.5 $\pm$ 0.3 & 9.2 $\pm$ 0.6 & 17 $\pm$ 1 & 0 \\
2 $\times$ 100 & 0.00015 & 97.9 $\pm$ 0.1 & 3.4 $\pm$ 0.3 & 12 $\pm$ 2 & 21 $\pm$ 4 & 0 \\
2 $\times$ 100 & 0.000175 & 97.88 $\pm$ 0.05 & 3.4 $\pm$ 0.3 & 15 $\pm$ 3 & 26 $\pm$ 4 & 0 \\
2 $\times$ 100 & 0.0002 & 97.91 $\pm$ 0.1 & 3.5 $\pm$ 0.4 & 18 $\pm$ 2 & 31 $\pm$ 3 & 0 \\
2 $\times$ 100 & 0.000225 & 97.8 $\pm$ 0.1 & 4.2 $\pm$ 0.9 & 18 $\pm$ 3 & 31 $\pm$ 5 & 0 \\
2 $\times$ 100 & 0.00025 & 97.65 $\pm$ 0.09 & 4 $\pm$ 0.5 & 20 $\pm$ 2 & 34 $\pm$ 4 & 0 \\
2 $\times$ 100 & 0.000275 & 97.69 $\pm$ 0.09 & 4 $\pm$ 1 & 22 $\pm$ 2 & 38 $\pm$ 3 & 0 \\
2 $\times$ 100 & 0.0003 & 97.64 $\pm$ 0.06 & 3.8 $\pm$ 0.4 & 24 $\pm$ 2 & 40 $\pm$ 4 & 0 \\
2 $\times$ 100 & 0.000325 & 97.52 $\pm$ 0.08 & 3.5 $\pm$ 0.3 & 24 $\pm$ 3 & 41 $\pm$ 4 & 0 \\
2 $\times$ 100 & 0.00035 & 97.42 $\pm$ 0.04 & 4 $\pm$ 1 & 23 $\pm$ 3 & 39 $\pm$ 4 & 0 \\
2 $\times$ 100 & 0.000375 & 97.3 $\pm$ 0.2 & 3.4 $\pm$ 0.3 & 24 $\pm$ 3 & 40 $\pm$ 5 & 0 \\
2 $\times$ 100 & 0.0004 & 97.28 $\pm$ 0.03 & 4.1 $\pm$ 0.7 & 23 $\pm$ 2 & 38 $\pm$ 3 & 0 \\
\midrule
\midrule
3 $\times$ 100 & 0 & 97.86 $\pm$ 0.06 & 3.9 $\pm$ 0.1 & 0 $\pm$ 0 & 0 $\pm$ 0 & 0 \\
3 $\times$ 100 & 0.000025 & 98.03 $\pm$ 0.08 & 10 $\pm$ 10 & 0 $\pm$ 0 & 0 $\pm$ 0 & 0 \\
3 $\times$ 100 & 0.00005 & 98.1 $\pm$ 0.1 & 20 $\pm$ 10 & 0.1 $\pm$ 0.3 & 0.2 $\pm$ 0.4 & 0 \\
3 $\times$ 100 & 0.000075 & 98.12 $\pm$ 0.07 & 20 $\pm$ 20 & 1.3 $\pm$ 0.7 & 1.8 $\pm$ 1 & 0 \\
3 $\times$ 100 & 0.0001 & 98.11 $\pm$ 0.09 & 8 $\pm$ 8 & 2.7 $\pm$ 0.9 & 4 $\pm$ 1 & 0 \\
3 $\times$ 100 & 0.000125 & 98.09 $\pm$ 0.1 & 2000 $\pm$ 4000 & 6 $\pm$ 1 & 11 $\pm$ 3 & 0 \\
3 $\times$ 100 & 0.00015 & 98.1 $\pm$ 0.1 & 100 $\pm$ 100 & 11 $\pm$ 2 & 20 $\pm$ 3 & 0 \\
3 $\times$ 100 & 0.000175 & 98.1 $\pm$ 0.1 & 70 $\pm$ 60 & 12 $\pm$ 2 & 20 $\pm$ 2 & 0 \\
3 $\times$ 100 & 0.0002 & 98 $\pm$ 0.1 & 20 $\pm$ 20 & 18 $\pm$ 2 & 30 $\pm$ 3 & 0 \\
\midrule
4 $\times$ 100 & 0 & 97.93 $\pm$ 0.07 & 4.2 $\pm$ 0.2 & 0 $\pm$ 0 & 0 $\pm$ 0 & 0 \\
4 $\times$ 100 & 0.000025 & 98 $\pm$ 0.1 & 200 $\pm$ 200 & 0 $\pm$ 0 & 0 $\pm$ 0 & 0 \\
4 $\times$ 100 & 0.00005 & 98.23 $\pm$ 0.08 & 1000 $\pm$ 3000 & 0.1 $\pm$ 0.1 & 0.1 $\pm$ 0.2 & 1 \\
4 $\times$ 100 & 0.000075 & 98.17 $\pm$ 0.09 & 1000 $\pm$ 1000 & 1.2 $\pm$ 0.4 & 1.5 $\pm$ 0.5 & 2 \\
4 $\times$ 100 & 0.0001 & 98.1 $\pm$ 0.06 & 3000 $\pm$ 3000 & 2.8 $\pm$ 0.9 & 4 $\pm$ 1 & 2 \\
4 $\times$ 100 & 0.00015 & 98.1 $\pm$ 0.2 & 2000 $\pm$ 1000 & 11 $\pm$ 2 & 20 $\pm$ 4 & 2 \\
4 $\times$ 100 & 0.000175 & 98.1 $\pm$ 0.1 & 1000 $\pm$ 2000 & 14 $\pm$ 1 & 24 $\pm$ 3 & 0 \\
4 $\times$ 100 & 0.0002 & 98.09 $\pm$ 0.07 & 1000 $\pm$ 1000 & 17 $\pm$ 2 & 30 $\pm$ 3 & 1 \\
\midrule
5 $\times$ 100 & 0 & 98.06 $\pm$ 0.03 & 2000 $\pm$ 3000 & 0 $\pm$ 0 & 0 $\pm$ 0 & 1 \\
5 $\times$ 100 & 0.000025 & 98.2 $\pm$ 0.1 & 1000 $\pm$ 100 & 0 $\pm$ 0 & 0 $\pm$ 0 & 3 \\
5 $\times$ 100 & 0.000175 & 98.1 $\pm$ 0.2 & 4000 $\pm$ 4000 & 15.1 $\pm$ 0.7 & 27 $\pm$ 2 & 3 \\
5 $\times$ 100 & 0.0002 & 98.1 $\pm$ 0.1 & 3000 $\pm$ 2000 & 18 $\pm$ 1 & 32 $\pm$ 2 & 1 \\
\bottomrule
\end{tabular}
\end{sc}
\end{small}
\end{center}
\vskip -0.1in
\end{table}

\begin{table}[!htb]
\caption{\textbf{MNIST Classifiers:} Compression results with fixed height and varying width.}
\label{tab:mnist_width}
\vskip 0.15in
\begin{center}
\begin{small}
\begin{sc}
\resizebox{\columnwidth}{!}{
\begin{tabular}{ccccccc}
\toprule
&&& Compression & \multicolumn{2}{c}{\% Removed} & Timed \\
Architecture & $\lOne$ & Accuracy (\%) & Runtime (s) & Neurons & Connections & Out \\
\midrule
2 $\times$ 100 & 0 & 97.92 $\pm$ 0.09 & 3.4 $\pm$ 0.3 & 0 $\pm$ 0 & 0 $\pm$ 0 & 0 \\
2 $\times$ 100 & 0.000025 & 97.93 $\pm$ 0.02 & 3.2 $\pm$ 0.1 & 0 $\pm$ 0 & 0 $\pm$ 0 & 0 \\
2 $\times$ 100 & 0.00005 & 98.06 $\pm$ 0.09 & 3.5 $\pm$ 0.3 & 0.1 $\pm$ 0.2 & 0.2 $\pm$ 0.4 & 0 \\
2 $\times$ 100 & 0.000075 & 98.13 $\pm$ 0.09 & 3.2 $\pm$ 0.2 & 1.1 $\pm$ 0.4 & 2 $\pm$ 0.8 & 0 \\
2 $\times$ 100 & 0.0001 & 98.12 $\pm$ 0.09 & 3.5 $\pm$ 0.1 & 3.4 $\pm$ 0.7 & 6 $\pm$ 1 & 0 \\
2 $\times$ 100 & 0.000125 & 98.01 $\pm$ 0.09 & 3.5 $\pm$ 0.3 & 9.2 $\pm$ 0.6 & 17 $\pm$ 1 & 0 \\
2 $\times$ 100 & 0.00015 & 97.9 $\pm$ 0.1 & 3.4 $\pm$ 0.3 & 12 $\pm$ 2 & 21 $\pm$ 4 & 0 \\
2 $\times$ 100 & 0.000175 & 97.88 $\pm$ 0.05 & 3.4 $\pm$ 0.3 & 15 $\pm$ 3 & 26 $\pm$ 4 & 0 \\
2 $\times$ 100 & 0.0002 & 97.91 $\pm$ 0.1 & 3.5 $\pm$ 0.4 & 18 $\pm$ 2 & 31 $\pm$ 3 & 0 \\
2 $\times$ 100 & 0.000225 & 97.8 $\pm$ 0.1 & 4.2 $\pm$ 0.9 & 18 $\pm$ 3 & 31 $\pm$ 5 & 0 \\
2 $\times$ 100 & 0.00025 & 97.65 $\pm$ 0.09 & 4 $\pm$ 0.5 & 20 $\pm$ 2 & 34 $\pm$ 4 & 0 \\
2 $\times$ 100 & 0.000275 & 97.69 $\pm$ 0.09 & 4 $\pm$ 1 & 22 $\pm$ 2 & 38 $\pm$ 3 & 0 \\
2 $\times$ 100 & 0.0003 & 97.64 $\pm$ 0.06 & 3.8 $\pm$ 0.4 & 24 $\pm$ 2 & 40 $\pm$ 4 & 0 \\
2 $\times$ 100 & 0.000325 & 97.52 $\pm$ 0.08 & 3.5 $\pm$ 0.3 & 24 $\pm$ 3 & 41 $\pm$ 4 & 0 \\
2 $\times$ 100 & 0.00035 & 97.42 $\pm$ 0.04 & 4 $\pm$ 1 & 23 $\pm$ 3 & 39 $\pm$ 4 & 0 \\
2 $\times$ 100 & 0.000375 & 97.3 $\pm$ 0.2 & 3.4 $\pm$ 0.3 & 24 $\pm$ 3 & 40 $\pm$ 5 & 0 \\
2 $\times$ 100 & 0.0004 & 97.28 $\pm$ 0.03 & 4.1 $\pm$ 0.7 & 23 $\pm$ 2 & 38 $\pm$ 3 & 0 \\
\midrule
\midrule
2 $\times$ 200 & 0 & 98.03 $\pm$ 0.05 & 6.9 $\pm$ 0.7 & 0 $\pm$ 0 & 0 $\pm$ 0 & 0 \\
2 $\times$ 200 & 0.000025 & 98.2 $\pm$ 0.05 & 7.1 $\pm$ 0.7 & 0 $\pm$ 0 & 0 $\pm$ 0 & 0 \\
2 $\times$ 200 & 0.00005 & 98.15 $\pm$ 0.04 & 7.2 $\pm$ 0.4 & 0.1 $\pm$ 0.1 & 0.2 $\pm$ 0.3 & 0 \\
2 $\times$ 200 & 0.000075 & 98.18 $\pm$ 0.09 & 12 $\pm$ 9 & 3 $\pm$ 1 & 6 $\pm$ 2 & 0 \\
2 $\times$ 200 & 0.0001 & 98.16 $\pm$ 0.07 & 8.8 $\pm$ 0.7 & 11 $\pm$ 1 & 20 $\pm$ 2 & 0 \\
2 $\times$ 200 & 0.000125 & 98.1 $\pm$ 0.09 & 14 $\pm$ 10 & 15 $\pm$ 2 & 26 $\pm$ 3 & 0 \\
2 $\times$ 200 & 0.00015 & 98 $\pm$ 0.02 & 10 $\pm$ 3 & 18 $\pm$ 2 & 32 $\pm$ 3 & 0 \\
2 $\times$ 200 & 0.000175 & 97.9 $\pm$ 0.1 & 9 $\pm$ 2 & 20 $\pm$ 2 & 35 $\pm$ 3 & 0 \\
2 $\times$ 200 & 0.0002 & 97.95 $\pm$ 0.08 & 8 $\pm$ 2 & 20.8 $\pm$ 0.6 & 36.6 $\pm$ 1 & 0 \\
\midrule
2 $\times$ 400 & 0 & 98.1 $\pm$ 0.1 & 14.8 $\pm$ 0.4 & 0 $\pm$ 0 & 0 $\pm$ 0 & 0 \\
2 $\times$ 400 & 0.000025 & 98.25 $\pm$ 0.09 & 14.5 $\pm$ 0.5 & 0 $\pm$ 0 & 0 $\pm$ 0 & 0 \\
2 $\times$ 400 & 0.00005 & 98.25 $\pm$ 0.07 & 20 $\pm$ 2 & 0 $\pm$ 0 & 0 $\pm$ 0 & 0 \\
2 $\times$ 400 & 0.000075 & 98.23 $\pm$ 0.07 & 180 $\pm$ 80 & 8 $\pm$ 1 & 16 $\pm$ 2 & 0 \\
2 $\times$ 400 & 0.0001 & 98.1 $\pm$ 0.09 & 200 $\pm$ 100 & 14 $\pm$ 1 & 26 $\pm$ 2 & 0 \\
2 $\times$ 400 & 0.000125 & 98.05 $\pm$ 0.08 & 50 $\pm$ 20 & 18 $\pm$ 1 & 32 $\pm$ 2 & 0 \\
2 $\times$ 400 & 0.00015 & 98.03 $\pm$ 0.05 & 29 $\pm$ 10 & 19 $\pm$ 2 & 34 $\pm$ 3 & 0 \\
2 $\times$ 400 & 0.000175 & 97.9 $\pm$ 0.1 & 100 $\pm$ 100 & 17.7 $\pm$ 0.8 & 32 $\pm$ 1 & 0 \\
2 $\times$ 400 & 0.0002 & 97.87 $\pm$ 0.1 & 1000 $\pm$ 1000 & 18 $\pm$ 1 & 33 $\pm$ 2 & 0 \\
\midrule
2 $\times$ 800 & 0 & 98.21 $\pm$ 0.05 & 37.6 $\pm$ 0.3 & 0 $\pm$ 0 & 0 $\pm$ 0 & 0 \\
2 $\times$ 800 & 0.000025 & 98.26 $\pm$ 0.05 & 38.2 $\pm$ 0.4 & 0 $\pm$ 0 & 0 $\pm$ 0 & 0 \\
2 $\times$ 800 & 0.000075 & 98.23 $\pm$ 0.03 & 1300 $\pm$ 800 & 12 $\pm$ 0.7 & 22 $\pm$ 1 & 0 \\
2 $\times$ 800 & 0.0001 & 98 $\pm$ 0.1 & 1000 $\pm$ 1000 & 15.9 $\pm$ 0.9 & 29 $\pm$ 1 & 0 \\
2 $\times$ 800 & 0.000125 & 98.01 $\pm$ 0.07 & 100 $\pm$ 100 & 16.8 $\pm$ 0.8 & 31 $\pm$ 1 & 0 \\
2 $\times$ 800 & 0.00015 & 98.07 $\pm$ 0.06 & 90 $\pm$ 30 & 17.3 $\pm$ 0.6 & 31 $\pm$ 1 & 0 \\
2 $\times$ 800 & 0.000175 & 97.91 $\pm$ 0.07 & 50 $\pm$ 20 & 16.5 $\pm$ 0.9 & 30 $\pm$ 2 & 0 \\
2 $\times$ 800 & 0.0002 & 97.78 $\pm$ 0.06 & 80 $\pm$ 30 & 16.7 $\pm$ 0.6 & 31 $\pm$ 1 & 0 \\
\bottomrule
\end{tabular}
}
\end{sc}
\end{small}
\end{center}
\vskip -0.1in
\end{table}

\textbf{Runtime Comparison with SoTA~~}
\myReferFigure{fig:mnist_runtime_with_reg} shows the comparison of runtimes with the proposed method and the baseline with the strength of $\lOne$ regularization on the MNIST classifiers. We observe that the new method presents a median gain of $\mathbf{81}$ times in speedup.

\begin{figure}[!tb]
    \centering
    \begin{subfigure}[t]{0.45\linewidth}
        \runtimeRegularizationPlot{0.7}{2}{300000}{0, 0.00005, 0.0001, 0.00015, 0.0002}{
            \addLinePlotWithErrorBar {mpl_red }{2pt}{data/mnist_runtime_fig2_sheet.txt}{100-100}{smooth}{runtime}{runtime-std}
            \addlegendentry{$2 \times 100$}
            
            \addLinePlotWithErrorBar {mpl_orange }{2pt}{data/mnist_runtime_fig2_sheet.txt}{200-200}{smooth}{runtime}{runtime-std}
            \addlegendentry{$2 \times 200$}
            
            \addLinePlotWithErrorBar {mpl_green  }{2pt}{data/mnist_runtime_fig2_sheet.txt}{400-400}{smooth}{runtime}{runtime-std}
            \addlegendentry{$2 \times 400$}
            
            \addLinePlotWithErrorBar {mpl_blue}{2pt}{data/mnist_runtime_fig2_sheet.txt}{800-800}{smooth}{runtime}{runtime-std}
            \addlegendentry{$2 \times 800$}

            \addLinePlotWithErrorBar {mpl_red }{2pt}{data/mnist_runtime_fig2_sheet.txt}{100-100}{dotted}{runtime-old}{runtime-old-std}
            \addLinePlotWithErrorBar {mpl_orange}{2pt}{data/mnist_runtime_fig2_sheet.txt}{200-200}{dotted}{runtime-old}{runtime-old-std}
            \addLinePlotWithErrorBar {mpl_green  }{2pt}{data/mnist_runtime_fig2_sheet.txt}{400-400}{dotted}{runtime-old}{runtime-old-std}
            \addLinePlotWithErrorBar {mpl_blue  }{2pt}{data/mnist_runtime_fig2_sheet.txt}{800-800}{dotted}{runtime-old}{runtime-old-std}

        }
        \caption{With width}
    \end{subfigure}
    \hfill
    \begin{subfigure}[t]{0.45\linewidth}
        \runtimeRegularizationPlot{0.7}{2}{300000}{0, 0.00005, 0.0001, 0.00015, 0.0002}{
            \addLinePlotWithErrorBar {mpl_red }{2pt}{data/mnist_runtime_fig2_sheet.txt}{100-100}{smooth}{runtime}{runtime-std}
            \addlegendentry{$2 \times 100$}
            \addLinePlotWithErrorBar {mpl_orange}{2pt}{data/mnist_runtime_fig2_sheet.txt}{100-100-100}{smooth}{runtime}{runtime-std}
            \addlegendentry{$3 \times 100$}
            \addLinePlotWithErrorBar {mpl_green }{2pt}{data/mnist_runtime_fig2_sheet.txt}{100-100-100-100}{smooth}{runtime}{runtime-std}
            \addlegendentry{$4 \times 100$}

            
            \addLinePlotWithErrorBar {mpl_red }{2pt}{data/mnist_runtime_fig2_sheet.txt}{100-100}{dotted}{runtime-old}{runtime-old-std}
            \addLinePlotWithErrorBar {mpl_orange}{2pt}{data/mnist_runtime_fig2_sheet.txt}{100-100-100}{dotted}{runtime-old}{runtime-old-std}
            \addLinePlotWithErrorBar {mpl_green }{2pt}{data/mnist_runtime_fig2_sheet.txt}{100-100-100-100}{dotted}{runtime-old}{runtime-old-std}
            
        }
        \caption{With depth}
    \end{subfigure}
    \caption{
    \textbf{MNIST Classifiers: Comparison of runtimes} for proposed method (solid) and baseline (dashed) with the strength of regularization to identify stable neurons: (a) with increasing width (b) with increasing depth. We report the average and the standard deviation of the runtime of models with five different initialization for each regularization. Note that the y-axis is in the log scale.
    The median speedup is $\mathbf{81}$ times.
    }
    \label{fig:mnist_runtime_with_reg}
\end{figure}

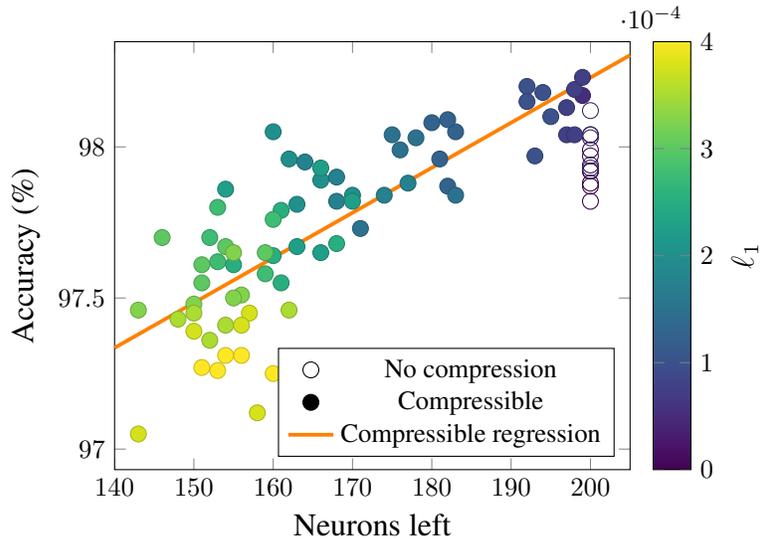
\begin{figure}[!htb]
    \centering
        \begin{tikzpicture}
        \begin{axis}[
            legend pos=south east,
            colorbar,
            point meta=explicit,
            point meta min=0,
            point meta max=0.0004,
            cycle list name=mark list,
            colormap name=viridis,
            xlabel={\large{Neurons left}},
            ylabel={\large{Accuracy (\%)}},
            colorbar style={ylabel={\large$\lOne$}},
            xmin=140,
            xmax=205,
        ]
        \addlegendentry{No compression}
        \addplot+[
            only marks,
            scatter,
            mark=o,
            mark size=3pt]
        table[meta=reg,x=nodes,y=accuracy]
        {data/mnist_uncompressed_regression.txt};
        \addplot+[
            only marks,
            scatter,
            mark=*,
            mark size=3pt]
        table[meta=reg,x=nodes,y=accuracy]
        {data/mnist_compressed_regression.txt};
        \addlegendentry{Compressible}
        \draw[color=orange,line width=0.5mm]
            (axis cs:140,97.335) -- 
            (axis cs:205,98.3035);
        \addlegendimage{line width=0.5mm,color=orange}
        \addlegendentry{Compressible regression}
        \end{axis}
        \end{tikzpicture}
    \caption{\textbf{Relationship between size of compressed neural network and accuracy on $2 \times 100$ MNIST classifiers.} The coefficient of determination (R$^2$) for the linear regression obtained for accuracy based on neurons left for compressible networks is $69$\%.}
    \label{fig:mnist_accuracy}
\end{figure}

\clearpage
\subsection{MNIST Autoencoders}\label{ap:autoencoder}

For the autoencoders, 
we use the notation $n_1~|~n_2~|~n_3$ for the architecture of $3$ hidden layers with $n_1, n_2,$ and $n_3$ neurons.
The output layer has the same size as the input, $784$, and uses ReLU activation. 
Starting with the architecture $100~|~10~|~100$, 
we evaluated changes to the bottleneck width $n_2$ as well as to the width of the other two layers. 
First, we changed the bottleneck width to $n_2 = 25$ and $n_2 = 50$. 
Second, we changed the width of the other layers to $n_1, n_3 = 50$, $n_1, n_3 = 200$, and $n_1, n_3 = 400$ while keeping $n_2 = 10$. 
For each architecture, 
we trained and evaluated neural networks with 5 different random initialization seeds using $\lOne = 0$, $\lOne = 0.00002$, and $\lOne = 0.0002$. 

\textbf{Relationship between Runtime and Regularization ~~}
\myReferTable{tab:ae_summary} reports the runtime to identify stable neurons and the proportion of neurons---as well as the corresponding connections---that can be removed due to stability in each case on MNIST Autoencoders. 

With the largest amount of regularization, 
we notice that the runtimes are considerably smaller and most of the network can be removed while the loss during training only doubles in comparison to using zero or a moderate amount of regularization. 
In fact, 
the only neurons that are not stable in such case are in the first layer, 
whereas between $3$ and $4$ out of the $5$ neural networks trained for each architecture have all hidden layers completely stable. 
By also evaluating the stability of the output layer, 
we identified a few cases in which the output layer is entirely stable. 
While we have not explicitly explored that possibility in the proposed algorithm, 
the implication for such case is that the neural network can be reduced to a linear function on the domain of interest. 
With autoencoders, we observed that this can happen when the regularization during training no more than doubles the loss, 
and that we can evaluate if that happens within seconds: the runtime when the stability of the output layer is tested is $1$ seconds on average and never more than $25$ seconds. 

\textbf{Runtime Comparison with SoTA~~}
\myReferFigure{fig:ae_runtime} shows the difference in runtimes between our approach and the baseline~\cite{serra2020lossless} for higher regularization, fixed $n_2=10$, and varying but equal values for $n_1$ and $n_3$ on the MNIST Autoencoders. We observe that the new method presents a median gain of $\mathbf{159}$ times in running time, which increases with the width of the non-bottleneck layers.

\begin{table*}[!htb]
\caption{\textbf{MNIST Autoencoders:} Compression results with varying architectures and levels of regularization.}
\label{tab:ae_summary}
\vskip 0.15in
\begin{center}
\begin{small}
\begin{sc}
\resizebox{\columnwidth}{!}{
\begin{tabular}{ccccccc}
\toprule
&&& Compression & \multicolumn{2}{c}{\% Removed} & Timed Out \\
Architecture & $\lOne$ & Loss & Runtime (s) & Neurons & Connections & ~ \\
\midrule
$100~|~10~|~100$ &0 &0.045 $\pm$ 0.001 & 130 $\pm$ 30 & 0.1 $\pm$ 0.1 & 0.05 $\pm$ 0.06 & 0 \\
$100~|~10~|~100$ &0.00002 &0.047 $\pm$ 0.0009 & 120 $\pm$ 30 & 12.7 $\pm$ 0.6 & 7.2 $\pm$ 0.9 & 0 \\
$100~|~10~|~100$ &0.0002 &0.077 $\pm$ 0.002 & 2.73 $\pm$ 0.05 & 95 $\pm$ 6 & 90 $\pm$ 10 & 0 \\ 
\midrule
$100~|~25~|~100$ &0 &0.035 $\pm$ 0.001 & 500 $\pm$ 300 & 0 $\pm$ 0 & 0 $\pm$ 0 & 0 \\
$100~|~25~|~100$ &0.00002 &0.047 $\pm$ 0.001 & 800 $\pm$ 200 & 14 $\pm$ 1 & 10 $\pm$ 2 & 0 \\
$100~|~25~|~100$ &0.0002 &0.076 $\pm$ 0.001 & 2.88 $\pm$ 0.08 & 90 $\pm$ 7 & 80 $\pm$ 20 & 0 \\
\midrule
$100~|~50~|~100$ &0 &0.0311 $\pm$ 0.0009 & 230 $\pm$ 20 & 0 $\pm$ 0 & 0 $\pm$ 0 & 0 \\
$100~|~50~|~100$ &0.00002 &0.0478 $\pm$ 0.0009 & 600 $\pm$ 200 & 17.4 $\pm$ 0.9 & 13 $\pm$ 1 & 0 \\
$100~|~50~|~100$ &0.0002 &0.081 $\pm$ 0.003 & 2.96 $\pm$ 0.04 & 90 $\pm$ 7 & 80 $\pm$ 20 & 0 \\
\midrule
\midrule
$50~|~10~|~50$ &0 &0.047 $\pm$ 0.002 & 33 $\pm$ 4 & 0 $\pm$ 0 & 0 $\pm$ 0 & 0 \\
$50~|~10~|~50$ &0.00002 &0.051 $\pm$ 0.002 & 50 $\pm$ 20 & 14 $\pm$ 3 & 13 $\pm$ 2 & 0 \\
$50~|~10~|~50$ &0.0002 &0.081 $\pm$ 0.002 & 1.42 $\pm$ 0.02 & 89 $\pm$ 8 & 88 $\pm$ 8 & 0 \\
\midrule
$100~|~10~|~100$ &0 &0.045 $\pm$ 0.001 & 130 $\pm$ 30 & 0.1 $\pm$ 0.1 & 0.05 $\pm$ 0.06 & 0 \\
$100~|~10~|~100$ &0.00002 &0.047 $\pm$ 0.0009 & 120 $\pm$ 30 & 12.7 $\pm$ 0.6 & 7.2 $\pm$ 0.9 & 0 \\
$100~|~10~|~100$ &0.0002 &0.077 $\pm$ 0.002 & 2.73 $\pm$ 0.05 & 95 $\pm$ 6 & 90 $\pm$ 10 & 0 \\
\midrule
$200~|~10~|~200$ &0 &0.041 $\pm$ 0.002 & 1000 $\pm$ 1000 & 0.4 $\pm$ 0.4 & 0.4 $\pm$ 0.4 & 1 \\
$200~|~10~|~200$ &0.00002 &0.043 $\pm$ 0.002 & 700 $\pm$ 400 & 14 $\pm$ 0.7 & 7 $\pm$ 1 & 0 \\
$200~|~10~|~200$ &0.0002 &0.076 $\pm$ 0.002 & 5.41 $\pm$ 0.03 & 95 $\pm$ 6 & 80 $\pm$ 20 & 0 \\
\midrule
$400~|~10~|~400$ &0 &0.04 & 2704 & 0 & 0 & 4 \\
$400~|~10~|~400$ &0.00002 &0.0395 $\pm$ 0.001 & 1300 $\pm$ 100 & 15 $\pm$ 1 & 6 $\pm$ 0.7 & 0 \\
$400~|~10~|~400$ &0.0002 &0.073 $\pm$ 0.001 & 10.5 $\pm$ 0.2 & 89.1 $\pm$ 7.5 & 13.6 $\pm$ 59.3 & 0 \\
\bottomrule
\end{tabular}
}
\end{sc}
\end{small}
\end{center}
\vskip -0.1in
\end{table*}

\begin{figure}[!tb]
    \centering
    \includegraphics[width=0.7\columnwidth]{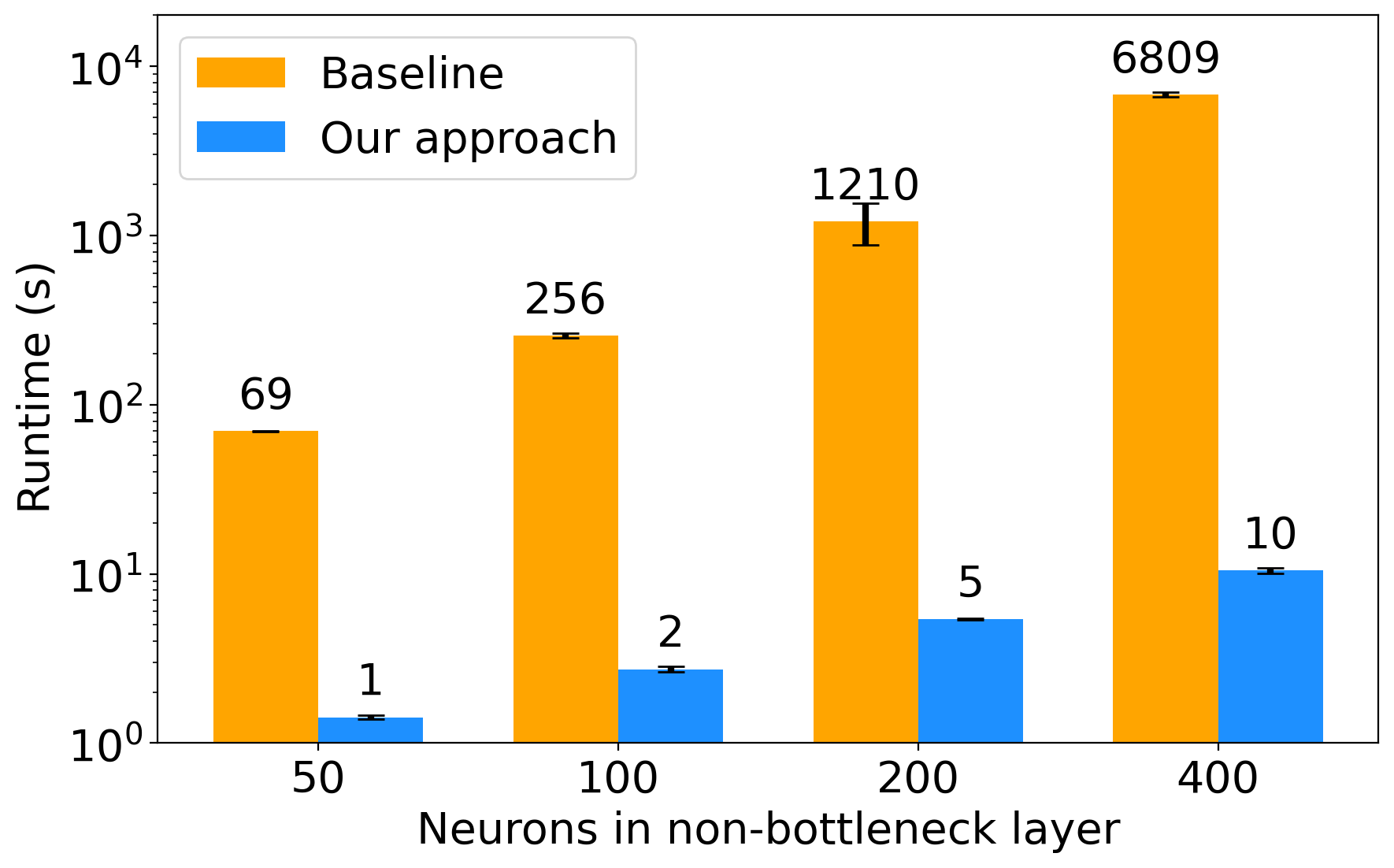}
    \caption{
    \textbf{MNIST Autoencoders: Comparison of runtimes} (in seconds) to identify stable neurons between the proposed approach vs. the baseline from~\cite{serra2020lossless} with high regularization ($\lOne = 0.0002$). Note that the y-axis is in the log scale. 
    The median speedup is $\mathbf{159}$ times.
    }
    \label{fig:ae_runtime}
\end{figure}

\clearpage
\subsection{CIFAR-10 Classifiers}

\textbf{Relationship between Runtime and Regularization ~~}
\myReferTable{tab:cifar10_depth} and \myReferTable{tab:cifar10_width} show the runtime achieved by the proposed method at different $\lOne$ regularization on the CIFAR-10 classifiers.

\begin{table}[!htb]
\caption{\textbf{CIFAR10 Classifiers:} Compression results with fixed width and varying depth.}
\label{tab:cifar10_depth}
\vskip 0.15in
\begin{center}
\begin{small}
\begin{sc}
\begin{tabular}{ccccccc}
\toprule
&&& Compression & \multicolumn{2}{c}{\% Removed} & Timed \\
Arch. & $\lOne$ & Accuracy (\%) & Runtime (s) & Neurons & Connections & Out \\
\midrule
2 $\times$ 100 & 0 & 54.3 $\pm$ 0.2 & 13.4 $\pm$ 0.6 & 0 $\pm$ 0 & 0 $\pm$ 0 & 0 \\
2 $\times$ 100 & 0.000025 & 53.8 $\pm$ 0.9 & 14 $\pm$ 2 & 0 $\pm$ 0 & 0 $\pm$ 0 & 0 \\
2 $\times$ 100 & 0.00005 & 53.6 $\pm$ 0.5 & 13 $\pm$ 3 & 31 $\pm$ 1 & 56 $\pm$ 2 & 0 \\
2 $\times$ 100 & 0.000075 & 52.7 $\pm$ 0.6 & 10.9 $\pm$ 0.8 & 34 $\pm$ 2 & 61 $\pm$ 4 & 0 \\
2 $\times$ 100 & 0.0001 & 52.3 $\pm$ 0.3 & 11 $\pm$ 2 & 36 $\pm$ 2 & 64 $\pm$ 2 & 0 \\
2 $\times$ 100 & 0.000125 & 51.6 $\pm$ 0.5 & 10.4 $\pm$ 0.3 & 39 $\pm$ 3 & 66 $\pm$ 4 & 0 \\
2 $\times$ 100 & 0.00015 & 51 $\pm$ 0.4 & 11 $\pm$ 2 & 40 $\pm$ 2 & 68 $\pm$ 3 & 0 \\
2 $\times$ 100 & 0.000175 & 50.4 $\pm$ 0.4 & 10.3 $\pm$ 0.1 & 42 $\pm$ 3 & 69 $\pm$ 3 & 0 \\
2 $\times$ 100 & 0.0002 & 50.1 $\pm$ 0.6 & 12 $\pm$ 2 & 45 $\pm$ 3 & 71 $\pm$ 3 & 0 \\
2 $\times$ 100 & 0.000225 & 49.6 $\pm$ 0.4 & 11 $\pm$ 1 & 45 $\pm$ 2 & 72 $\pm$ 1 & 0 \\
2 $\times$ 100 & 0.00025 & 48.5 $\pm$ 0.3 & 10.8 $\pm$ 0.7 & 46 $\pm$ 1 & 73 $\pm$ 2 & 0 \\
2 $\times$ 100 & 0.000275 & 48 $\pm$ 0.4 & 10.3 $\pm$ 0.2 & 47 $\pm$ 3 & 75 $\pm$ 3 & 0 \\
2 $\times$ 100 & 0.0003 & 47.8 $\pm$ 0.6 & 10.7 $\pm$ 0.6 & 51 $\pm$ 2 & 78 $\pm$ 2 & 0 \\
2 $\times$ 100 & 0.000325 & 47.2 $\pm$ 0.2 & 10.4 $\pm$ 0.2 & 51 $\pm$ 3 & 77 $\pm$ 2 & 0 \\
2 $\times$ 100 & 0.00035 & 47.2 $\pm$ 0.3 & 10.5 $\pm$ 0.5 & 53 $\pm$ 3 & 79 $\pm$ 3 & 0 \\
2 $\times$ 100 & 0.000375 & 46.8 $\pm$ 0.4 & 10.7 $\pm$ 0.5 & 54 $\pm$ 3 & 80 $\pm$ 2 & 0 \\
2 $\times$ 100 & 0.0004 & 46.3 $\pm$ 0.3 & 10.9 $\pm$ 0.4 & 56 $\pm$ 2 & 81 $\pm$ 2 & 0 \\
\midrule
\midrule
3 $\times$ 100 & 0 & 53.7 $\pm$ 0.7 & 13 $\pm$ 1 & 0 $\pm$ 0 & 0 $\pm$ 0 & 0 \\
3 $\times$ 100 & 0.000025 & 54.5 $\pm$ 0.4 & 20 $\pm$ 10 & 0 $\pm$ 0 & 0 $\pm$ 0 & 0 \\
3 $\times$ 100 & 0.00005 & 53.8 $\pm$ 0.4 & 13 $\pm$ 2 & 22.3 $\pm$ 0.8 & 32 $\pm$ 1 & 0 \\
3 $\times$ 100 & 0.000075 & 53.3 $\pm$ 0.6 & 11.6 $\pm$ 0.9 & 23 $\pm$ 2 & 34 $\pm$ 4 & 0 \\
3 $\times$ 100 & 0.0001 & 53.2 $\pm$ 0.6 & 20 $\pm$ 10 & 25 $\pm$ 2 & 36 $\pm$ 3 & 0 \\
3 $\times$ 100 & 0.000125 & 52.5 $\pm$ 0.6 & 14 $\pm$ 5 & 26 $\pm$ 2 & 38 $\pm$ 3 & 0 \\
3 $\times$ 100 & 0.00015 & 51.98 $\pm$ 0.05 & 16 $\pm$ 6 & 29 $\pm$ 1 & 43 $\pm$ 1 & 0 \\
3 $\times$ 100 & 0.000175 & 50.8 $\pm$ 0.6 & 12 $\pm$ 1 & 32 $\pm$ 2 & 47 $\pm$ 2 & 0 \\
3 $\times$ 100 & 0.0002 & 50.3 $\pm$ 0.4 & 15 $\pm$ 7 & 35 $\pm$ 2 & 52 $\pm$ 3 & 0 \\
\midrule
4 $\times$ 100 & 0 & 53.6 $\pm$ 0.6 & 20 $\pm$ 10 & 0 $\pm$ 0 & 0 $\pm$ 0 & 0 \\
4 $\times$ 100 & 0.000025 & 53.9 $\pm$ 0.6 & 20 $\pm$ 20 & 0 $\pm$ 0 & 0 $\pm$ 0 & 0 \\
4 $\times$ 100 & 0.00005 & 53.9 $\pm$ 0.2 & 17 $\pm$ 8 & 15.9 $\pm$ 0.6 & 20.5 $\pm$ 0.8 & 0 \\
4 $\times$ 100 & 0.000075 & 53.7 $\pm$ 0.3 & 13 $\pm$ 1 & 17 $\pm$ 1 & 22 $\pm$ 1 & 0 \\
4 $\times$ 100 & 0.0001 & 52.7 $\pm$ 0.3 & 60 $\pm$ 90 & 19.3 $\pm$ 1 & 25 $\pm$ 1 & 0 \\
4 $\times$ 100 & 0.000125 & 52.4 $\pm$ 0.6 & 15 $\pm$ 5 & 21 $\pm$ 2 & 29 $\pm$ 2 & 0 \\
4 $\times$ 100 & 0.00015 & 51.6 $\pm$ 0.2 & 600 $\pm$ 800 & 25 $\pm$ 1 & 34 $\pm$ 2 & 0 \\
4 $\times$ 100 & 0.000175 & 50.7 $\pm$ 0.3 & 700 $\pm$ 800 & 28.5 $\pm$ 0.9 & 40 $\pm$ 1 & 1 \\
4 $\times$ 100 & 0.0002 & 50.3 $\pm$ 0.4 & 400 $\pm$ 400 & 33.7 $\pm$ 0.9 & 49 $\pm$ 1 & 0 \\
\midrule
5 $\times$ 100 & 0 & 53 $\pm$ 0.5 & 14.4 $\pm$ 0.4 & 0 $\pm$ 0 & 0 $\pm$ 0 & 0 \\
5 $\times$ 100 & 0.000025 & 53.3 $\pm$ 0.8 & 18 $\pm$ 5 & 0 $\pm$ 0 & 0 $\pm$ 0 & 0 \\
5 $\times$ 100 & 0.00005 & 54 $\pm$ 0.1 & 30 $\pm$ 20 & 12.9 $\pm$ 0.6 & 15.7 $\pm$ 0.7 & 0 \\
5 $\times$ 100 & 0.000075 & 53.5 $\pm$ 0.4 & 100 $\pm$ 200 & 14 $\pm$ 0.5 & 17.1 $\pm$ 0.6 & 0 \\
5 $\times$ 100 & 0.0001 & 53.3 $\pm$ 0.3 & 11.8 $\pm$ 0.4 & 16 $\pm$ 1 & 20 $\pm$ 1 & 2 \\
5 $\times$ 100 & 0.000125 & 51.9 $\pm$ 0.4 & 3000 $\pm$ 4000 & 14 $\pm$ 8 & 20 $\pm$ 10 & 2 \\
5 $\times$ 100 & 0.00015 & 51.4 & 1000 & 20 & 27 & 4 \\
5 $\times$ 100 & 0.000175 & 51.3 $\pm$ 0.4 & 2000 $\pm$ 2000 & 27.4 $\pm$ 0.8 & 39 $\pm$ 1 & 3 \\
5 $\times$ 100 & 0.0002 & 50.2 $\pm$ 0.1 & 3000 $\pm$ 2000 & 31 $\pm$ 2 & 45 $\pm$ 3 & 1 \\
\bottomrule
\end{tabular}
\end{sc}
\end{small}
\end{center}
\vskip -0.1in
\end{table}

\begin{table}[!htb]
\caption{\textbf{CIFAR10 Classifiers:} Compression results with fixed height and varying width.}
\label{tab:cifar10_width}
\vskip 0.15in
\begin{center}
\begin{small}
\begin{sc}
\resizebox{\columnwidth}{!}{
\begin{tabular}{ccccccc}
\toprule
&&& Compression & \multicolumn{2}{c}{\% Removed} & Timed \\
Architecture & $\lOne$ & Accuracy (\%) & Runtime (s) & Neurons & Connections & Out \\
\midrule
2 $\times$ 100 & 0 & 54.3 $\pm$ 0.2 & 13.4 $\pm$ 0.6 & 0 $\pm$ 0 & 0 $\pm$ 0 & 0 \\
2 $\times$ 100 & 0.000025 & 53.8 $\pm$ 0.9 & 14 $\pm$ 2 & 0 $\pm$ 0 & 0 $\pm$ 0 & 0 \\
2 $\times$ 100 & 0.00005 & 53.6 $\pm$ 0.5 & 13 $\pm$ 3 & 31 $\pm$ 1 & 56 $\pm$ 2 & 0 \\
2 $\times$ 100 & 0.000075 & 52.7 $\pm$ 0.6 & 10.9 $\pm$ 0.8 & 34 $\pm$ 2 & 61 $\pm$ 4 & 0 \\
2 $\times$ 100 & 0.0001 & 52.3 $\pm$ 0.3 & 11 $\pm$ 2 & 36 $\pm$ 2 & 64 $\pm$ 2 & 0 \\
2 $\times$ 100 & 0.000125 & 51.6 $\pm$ 0.5 & 10.4 $\pm$ 0.3 & 39 $\pm$ 3 & 66 $\pm$ 4 & 0 \\
2 $\times$ 100 & 0.00015 & 51 $\pm$ 0.4 & 11 $\pm$ 2 & 40 $\pm$ 2 & 68 $\pm$ 3 & 0 \\
2 $\times$ 100 & 0.000175 & 50.4 $\pm$ 0.4 & 10.3 $\pm$ 0.1 & 42 $\pm$ 3 & 69 $\pm$ 3 & 0 \\
2 $\times$ 100 & 0.0002 & 50.1 $\pm$ 0.6 & 12 $\pm$ 2 & 45 $\pm$ 3 & 71 $\pm$ 3 & 0 \\
2 $\times$ 100 & 0.000225 & 49.6 $\pm$ 0.4 & 11 $\pm$ 1 & 45 $\pm$ 2 & 72 $\pm$ 1 & 0 \\
2 $\times$ 100 & 0.00025 & 48.5 $\pm$ 0.3 & 10.8 $\pm$ 0.7 & 46 $\pm$ 1 & 73 $\pm$ 2 & 0 \\
2 $\times$ 100 & 0.000275 & 48 $\pm$ 0.4 & 10.3 $\pm$ 0.2 & 47 $\pm$ 3 & 75 $\pm$ 3 & 0 \\
2 $\times$ 100 & 0.0003 & 47.8 $\pm$ 0.6 & 10.7 $\pm$ 0.6 & 51 $\pm$ 2 & 78 $\pm$ 2 & 0 \\
2 $\times$ 100 & 0.000325 & 47.2 $\pm$ 0.2 & 10.4 $\pm$ 0.2 & 51 $\pm$ 3 & 77 $\pm$ 2 & 0 \\
2 $\times$ 100 & 0.00035 & 47.2 $\pm$ 0.3 & 10.5 $\pm$ 0.5 & 53 $\pm$ 3 & 79 $\pm$ 3 & 0 \\
2 $\times$ 100 & 0.000375 & 46.8 $\pm$ 0.4 & 10.7 $\pm$ 0.5 & 54 $\pm$ 3 & 80 $\pm$ 2 & 0 \\
2 $\times$ 100 & 0.0004 & 46.3 $\pm$ 0.3 & 10.9 $\pm$ 0.4 & 56 $\pm$ 2 & 81 $\pm$ 2 & 0 \\
\midrule
\midrule
2 $\times$ 200 & 0 & 56.8 $\pm$ 0.2 & 23 $\pm$ 2 & 0 $\pm$ 0 & 0 $\pm$ 0 & 0 \\
2 $\times$ 200 & 0.000025 & 56.8 $\pm$ 0.6 & 28 $\pm$ 1 & 0 $\pm$ 0 & 0 $\pm$ 0 & 0 \\
2 $\times$ 200 & 0.00005 & 56.3 $\pm$ 0.4 & 30 $\pm$ 10 & 28 $\pm$ 2 & 54 $\pm$ 3 & 0 \\
2 $\times$ 200 & 0.000075 & 55.5 $\pm$ 0.3 & 40 $\pm$ 20 & 32 $\pm$ 2 & 61 $\pm$ 3 & 0 \\
2 $\times$ 200 & 0.0001 & 54.3 $\pm$ 0.4 & 24 $\pm$ 6 & 37 $\pm$ 2 & 68 $\pm$ 3 & 0 \\
2 $\times$ 200 & 0.000125 & 53.3 $\pm$ 0.3 & 1000 $\pm$ 2000 & 42 $\pm$ 1 & 72 $\pm$ 2 & 0 \\
2 $\times$ 200 & 0.00015 & 51.9 $\pm$ 0.7 & 24 $\pm$ 4 & 45 $\pm$ 2 & 75 $\pm$ 2 & 0 \\
2 $\times$ 200 & 0.000175 & 51.2 $\pm$ 0.4 & 21.6 $\pm$ 0.8 & 49 $\pm$ 2 & 78 $\pm$ 2 & 0 \\
2 $\times$ 200 & 0.0002 & 50.4 $\pm$ 0.1 & 23 $\pm$ 3 & 52 $\pm$ 2 & 80 $\pm$ 1 & 0 \\
\midrule
2 $\times$ 400 & 0 & 58.7 $\pm$ 0.1 & 48 $\pm$ 2 & 0 $\pm$ 0 & 0 $\pm$ 0 & 0 \\
2 $\times$ 400 & 0.000025 & 59.2 $\pm$ 0.4 & 55 $\pm$ 9 & 0 $\pm$ 0 & 0 $\pm$ 0 & 0 \\
2 $\times$ 400 & 0.00005 & 58.2 $\pm$ 0.1 & 60 $\pm$ 30 & 28 $\pm$ 1 & 54 $\pm$ 2 & 0 \\
2 $\times$ 400 & 0.000075 & 56.1 $\pm$ 0.2 & 51 $\pm$ 3 & 37 $\pm$ 1 & 68 $\pm$ 2 & 0 \\
2 $\times$ 400 & 0.0001 & 55 $\pm$ 0.3 & 48 $\pm$ 4 & 45 $\pm$ 2 & 75 $\pm$ 2 & 0 \\
2 $\times$ 400 & 0.000125 & 53.5 $\pm$ 0.2 & 45 $\pm$ 3 & 48.3 $\pm$ 0.8 & 77.5 $\pm$ 0.6 & 0 \\
2 $\times$ 400 & 0.00015 & 51.9 $\pm$ 0.3 & 50 $\pm$ 10 & 52 $\pm$ 1 & 80 $\pm$ 2 & 0 \\
2 $\times$ 400 & 0.000175 & 50.9 $\pm$ 0.5 & 43 $\pm$ 3 & 56 $\pm$ 2 & 83 $\pm$ 1 & 0 \\
2 $\times$ 400 & 0.0002 & 50.3 $\pm$ 0.3 & 45 $\pm$ 3 & 58 $\pm$ 3 & 83 $\pm$ 2 & 0 \\
\midrule
2 $\times$ 800 & 0 & 60.3 $\pm$ 0.2 & 125 $\pm$ 7 & 0 $\pm$ 0 & 0 $\pm$ 0 & 0 \\
2 $\times$ 800 & 0.000025 & 60.3 $\pm$ 0.2 & 190 $\pm$ 80 & 0 $\pm$ 0 & 0 $\pm$ 0 & 0 \\
2 $\times$ 800 & 0.00005 & 58.3 $\pm$ 0.2 & 240 $\pm$ 90 & 23 $\pm$ 6 & 50 $\pm$ 10 & 0 \\
2 $\times$ 800 & 0.000075 & 56.3 $\pm$ 0.3 & 150 $\pm$ 50 & 37 $\pm$ 9 & 60 $\pm$ 10 & 0 \\
2 $\times$ 800 & 0.0001 & 54.6 $\pm$ 0.2 & 108 $\pm$ 9 & 40 $\pm$ 10 & 70 $\pm$ 10 & 0 \\
2 $\times$ 800 & 0.000125 & 53.2 $\pm$ 0.5 & 130 $\pm$ 30 & 50 $\pm$ 2 & 76 $\pm$ 2 & 0 \\
2 $\times$ 800 & 0.00015 & 51.8 $\pm$ 0.3 & 110 $\pm$ 10 & 52.5 $\pm$ 0.8 & 78.2 $\pm$ 0.7 & 0 \\
2 $\times$ 800 & 0.000175 & 50.6 $\pm$ 0.4 & 99 $\pm$ 6 & 53 $\pm$ 1 & 78.7 $\pm$ 1 & 0 \\
2 $\times$ 800 & 0.0002 & 50.3 $\pm$ 0.2 & 98 $\pm$ 6 & 54 $\pm$ 1 & 79 $\pm$ 1 & 0 \\
\bottomrule
\end{tabular}
}
\end{sc}
\end{small}
\end{center}
\vskip -0.1in
\end{table}

\textbf{Runtime Comparison with SoTA~~}
\myReferFigure{fig:cifar_runtime_with_reg} shows the comparison of runtime of the proposed method and the baseline with the strength of $\lOne$ regularization on the CIFAR-10 classifiers. We observe that the new method presents a median gain of $\mathbf{183}$ times in running time.

\begin{figure}[!t]
    \centering
    \begin{subfigure}[t]{0.45\linewidth}
        \runtimeRegularizationPlot{0.7}{2}{300000}{0, 0.00005, 0.0001, 0.00015, 0.0002}{
            \addLinePlotWithErrorBar {mpl_red }{2pt}{data/cifar/cifar_figure1_sheet.txt}{100-100}{smooth}{runtime}{runtime-std}
            \addlegendentry{$2 \times 100$}
            
            \addLinePlotWithErrorBar {mpl_orange }{2pt}{data/cifar/cifar_figure1_sheet.txt}{200-200}{smooth}{runtime}{runtime-std}
            \addlegendentry{$2 \times 200$}
            
            \addLinePlotWithErrorBar {mpl_green  }{2pt}{data/cifar/cifar_figure1_sheet.txt}{400-400}{smooth}{runtime}{runtime-std}
            \addlegendentry{$2 \times 400$}
            
            \addLinePlotWithErrorBar {mpl_blue}{2pt}{data/cifar/cifar_figure1_sheet.txt}{800-800}{smooth}{runtime}{runtime-std}
            \addlegendentry{$2 \times 800$}

            \addLinePlotWithErrorBar {mpl_red }{2pt}{data/cifar/cifar_figure1_sheet.txt}{100-100}{dotted}{runtime-old}{runtime-old-std}
            \addLinePlotWithErrorBar {mpl_orange }{2pt}{data/cifar/cifar_figure1_sheet.txt}{200-200}{dotted}{runtime-old}{runtime-old-std}
            \addLinePlotWithErrorBar {mpl_green  }{2pt}{data/cifar/cifar_figure1_sheet.txt}{400-400}{dotted}{runtime-old}{runtime-old-std}
            \addLinePlotWithErrorBar {mpl_blue}{1pt}{data/cifar/cifar_figure1_sheet.txt}{800-800}{dotted}{runtime-old}{runtime-old-std}

        }
        \caption{With width}
    \end{subfigure}
    \hfill
    \begin{subfigure}[t]{0.45\linewidth}
        \runtimeRegularizationPlot{0.7}{2}{300000}{0, 0.00005, 0.0001, 0.00015, 0.0002}{
            \addLinePlotWithErrorBar {mpl_red }{2pt}{data/cifar/cifar_figure1_sheet.txt}{100-100}{smooth}{runtime}{runtime-std}
            \addlegendentry{$2 \times 100$}
            \addLinePlotWithErrorBar {mpl_orange}{2pt}{data/cifar/cifar_figure1_sheet.txt}{100-100-100}{smooth}{runtime}{runtime-std}
            \addlegendentry{$3 \times 100$}
    
            \addLinePlotWithErrorBar {mpl_red }{2pt}{data/cifar/cifar_figure1_sheet.txt}{100-100}{dotted}{runtime-old}{runtime-old-std}
            \addLinePlotWithErrorBar {mpl_orange}{2pt}{data/cifar/cifar_figure1_sheet.txt}{100-100-100}{dotted}{runtime-old}{runtime-old-std}
            
        }
        \caption{With depth}
    \end{subfigure}
    \caption{
    \textbf{CIFAR-10 Classifiers: Comparison of runtimes} for proposed method (solid) and baseline (dashed) with the strength of regularization to identify stable neurons: (a) with increasing width (b) with increasing depth. We report the average and the standard deviation of the runtime of models with five different initialization for each regularization. Note that the y-axis is in the log scale.
    The median speedup is $\mathbf{183}$ times. 
    }
    \label{fig:cifar_runtime_with_reg}
\end{figure}

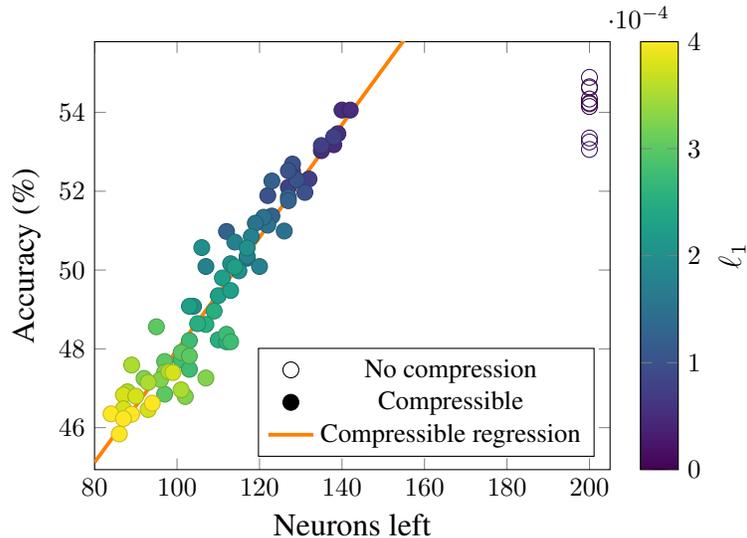
\begin{figure}[!htb]
    \centering
        \begin{tikzpicture}
        \begin{axis}[
            legend pos=south east,
            colorbar,
            point meta=explicit,
            point meta min=0,
            point meta max=0.0004,
            cycle list name=mark list,
            colormap name=viridis,
            xlabel={\large{Neurons left}},
            ylabel={\large{Accuracy (\%)}},
            colorbar style={ylabel={\large$\lOne$}},
            xmin=80,
            xmax=205,
        ]
        \addlegendentry{No compression}
        \addplot+[
            only marks,
            scatter,
            mark=o,
            mark size=3pt]
        table[meta=reg,x=nodes,y=accuracy]
        {data/cifar10_uncompressed_regression.txt};
        \addplot+[
            only marks,
            scatter,
            mark=*,
            mark size=3pt]
        table[meta=reg,x=nodes,y=accuracy]
        {data/cifar10_compressed_regression.txt};
        \addlegendentry{Compressible}
        \draw[color=orange,line width=0.5mm]
            (axis cs:80,45.12) -- 
            (axis cs:205,62.97);
        \addlegendimage{line width=0.5mm,color=orange}
        \addlegendentry{Compressible regression}
        \end{axis}
        \end{tikzpicture}
    \caption{\textbf{Relationship between size of compressed neural network and accuracy on $2 \times 100$ CIFAR-10 classifiers.} The coefficient of determination (R$^2$) for the linear regression obtained for accuracy based on neurons left for compressible networks is $91$\%.}
    \label{fig:mnist_accuracy}
\end{figure}

\clearpage
\subsection{CIFAR-100 Classifiers}

\textbf{Relationship between Runtime and Regularization ~~}
\myReferTable{tab:cifar100_depth} and \myReferTable{tab:cifar100_width} show the runtime achieved by the proposed method at different $\lOne$ regularization on the CIFAR-100 classifiers.

\begin{table}[!htb]
\caption{\textbf{CIFAR100 Classifiers:} Compression results with fixed width and varying depth.}
\label{tab:cifar100_depth}
\vskip 0.15in
\begin{center}
\begin{small}
\begin{sc}
\begin{tabular}{ccccccc}
\toprule
&&& Compression & \multicolumn{2}{c}{\% Removed} & Timed \\
Arch. & $\lOne$ & Accuracy (\%) & Runtime (s) & Neurons & Connections & Out \\
\midrule
2 $\times$ 100 & 0 & 25.2 $\pm$ 0.2 & 13 $\pm$ 1 & 0 $\pm$ 0 & 0 $\pm$ 0 & 0 \\
2 $\times$ 100 & 0.000025 & 24.8 $\pm$ 0.4 & 13 $\pm$ 2 & 0 $\pm$ 0 & 0 $\pm$ 0 & 1 \\
2 $\times$ 100 & 0.00005 & 24 $\pm$ 0.7 & 11.4 $\pm$ 0.4 & 36 $\pm$ 4 & 36 $\pm$ 4 & 1 \\
2 $\times$ 100 & 0.000075 & 23.7 & 10.4 & 42 & 42 & 4 \\
2 $\times$ 100 & 0.0001 & 23.4 $\pm$ 0.6 & 10.3 $\pm$ 0.1 & 42 $\pm$ 1 & 43 $\pm$ 1 & 1 \\
2 $\times$ 100 & 0.000125 & 22 $\pm$ 2 & 10.453 $\pm$ 0.004 & 48 $\pm$ 1 & 48 $\pm$ 2 & 3 \\
2 $\times$ 100 & 0.00015 & 22.4 $\pm$ 0.6 & 10.8 $\pm$ 1 & 48 $\pm$ 3 & 48 $\pm$ 3 & 1 \\
2 $\times$ 100 & 0.000175 & 21.5 $\pm$ 0.5 & 10.8 $\pm$ 0.3 & 47.9 $\pm$ 0.2 & 48.7 $\pm$ 0.4 & 1 \\
2 $\times$ 100 & 0.0002 & 21 $\pm$ 1 & 10.6 $\pm$ 0.3 & 51 $\pm$ 2 & 52 $\pm$ 2 & 0 \\
2 $\times$ 100 & 0.000225 & 21.2 $\pm$ 0.4 & 11 $\pm$ 0.7 & 51 $\pm$ 2 & 52 $\pm$ 2 & 0 \\
2 $\times$ 100 & 0.00025 & 21 $\pm$ 2 & 10.6 $\pm$ 0.5 & 50 $\pm$ 3 & 52 $\pm$ 4 & 0 \\
2 $\times$ 100 & 0.000275 & 20.7 $\pm$ 0.8 & 10.4 $\pm$ 0.1 & 52 $\pm$ 2 & 54 $\pm$ 2 & 0 \\
2 $\times$ 100 & 0.0003 & 19 $\pm$ 1 & 10.6 $\pm$ 0.2 & 53 $\pm$ 2 & 55 $\pm$ 2 & 0 \\
2 $\times$ 100 & 0.000325 & 19 $\pm$ 1 & 10.7 $\pm$ 0.7 & 53 $\pm$ 4 & 55 $\pm$ 4 & 0 \\
2 $\times$ 100 & 0.00035 & 19.2 $\pm$ 0.9 & 11 $\pm$ 1 & 53 $\pm$ 2 & 55 $\pm$ 1 & 0 \\
2 $\times$ 100 & 0.000375 & 19.4 $\pm$ 0.5 & 10.5 $\pm$ 0.4 & 54 $\pm$ 2 & 56 $\pm$ 2 & 0 \\
2 $\times$ 100 & 0.0004 & 19 $\pm$ 0.5 & 10.5 $\pm$ 0.3 & 53 $\pm$ 3 & 56 $\pm$ 3 & 0 \\
\midrule
\midrule
3 $\times$ 100 & 0 & 24.9 $\pm$ 0.4 & 16 $\pm$ 3 & 0 $\pm$ 0 & 0 $\pm$ 0 & 0 \\
3 $\times$ 100 & 0.000025 & 25.1 $\pm$ 0.4 & 17 $\pm$ 2 & 0 $\pm$ 0 & 0 $\pm$ 0 & 2 \\
3 $\times$ 100 & 0.00005 & 25.4 $\pm$ 0.6 & 20 $\pm$ 10 & 22 $\pm$ 2 & 22 $\pm$ 2 & 2 \\
3 $\times$ 100 & 0.000075 & 24 $\pm$ 1 & 13 $\pm$ 3 & 28 $\pm$ 2 & 28 $\pm$ 2 & 1 \\
3 $\times$ 100 & 0.0001 & 24 $\pm$ 1 & 11.3 $\pm$ 0.4 & 30 $\pm$ 0.9 & 30.4 $\pm$ 1 & 1 \\
3 $\times$ 100 & 0.000125 & 24 $\pm$ 1 & 12 $\pm$ 1 & 31 $\pm$ 1 & 32.4 $\pm$ 0.9 & 1 \\
3 $\times$ 100 & 0.00015 & 23.1 $\pm$ 0.5 & 50 $\pm$ 80 & 34 $\pm$ 1 & 37 $\pm$ 1 & 0 \\
3 $\times$ 100 & 0.000175 & 22 $\pm$ 1 & 10.7 $\pm$ 0.4 & 36 $\pm$ 2 & 38 $\pm$ 3 & 0 \\
3 $\times$ 100 & 0.0002 & 22.4 $\pm$ 0.6 & 12 $\pm$ 1 & 39 $\pm$ 2 & 44 $\pm$ 3 & 0 \\
\midrule
4 $\times$ 100 & 0 & 24.7 $\pm$ 0.5 & 30 $\pm$ 20 & 0 $\pm$ 0 & 0 $\pm$ 0 & 0 \\
4 $\times$ 100 & 0.000025 & 25 $\pm$ 0.7 & 16 $\pm$ 4 & 0 $\pm$ 0 & 0 $\pm$ 0 & 1 \\
4 $\times$ 100 & 0.00005 & 24.8 $\pm$ 0.8 & 2000 $\pm$ 3000 & 18 $\pm$ 1 & 18 $\pm$ 1 & 1 \\
4 $\times$ 100 & 0.000075 & 25.1 $\pm$ 0.5 & 12 $\pm$ 1 & 20 $\pm$ 1 & 20 $\pm$ 1 & 1 \\
4 $\times$ 100 & 0.0001 & 24.8 $\pm$ 0.2 & 12 $\pm$ 2 & 22 $\pm$ 2 & 22 $\pm$ 2 & 2 \\
4 $\times$ 100 & 0.000125 & 23.9 $\pm$ 0.4 & 11.8 $\pm$ 0.5 & 23.9 $\pm$ 0.4 & 25 $\pm$ 0.7 & 2 \\
4 $\times$ 100 & 0.00015 & 23 $\pm$ 1 & 50 $\pm$ 70 & 28 $\pm$ 2 & 31 $\pm$ 3 & 1 \\
4 $\times$ 100 & 0.000175 & 22 $\pm$ 2 & 50 $\pm$ 60 & 31 $\pm$ 3 & 36 $\pm$ 4 & 0 \\
4 $\times$ 100 & 0.0002 & 22 $\pm$ 1 & 100 $\pm$ 200 & 34 $\pm$ 2 & 41 $\pm$ 2 & 0 \\
\midrule
5 $\times$ 100 & 0 & 24.2 $\pm$ 0.5 & 18 $\pm$ 4 & 0 $\pm$ 0 & 0 $\pm$ 0 & 0 \\
5 $\times$ 100 & 0.000025 & 24.6 $\pm$ 0.4 & 100 $\pm$ 200 & 0 $\pm$ 0 & 0 $\pm$ 0 & 0 \\
5 $\times$ 100 & 0.00005 & 25.4 $\pm$ 0.1 & 40 $\pm$ 30 & 12.9 $\pm$ 0.7 & 12.9 $\pm$ 0.7 & 3 \\
5 $\times$ 100 & 0.000075 & 24.6 $\pm$ 0.2 & 14.1 $\pm$ 0.4 & 16.4 $\pm$ 0.3 & 16.6 $\pm$ 0.3 & 2 \\
5 $\times$ 100 & 0.0001 & 24 $\pm$ 1 & 1000 $\pm$ 2000 & 18 $\pm$ 1 & 19 $\pm$ 2 & 1 \\
5 $\times$ 100 & 0.000125 & 24.3 $\pm$ 0.2 & 200 $\pm$ 300 & 19 $\pm$ 1 & 20 $\pm$ 1 & 2 \\
5 $\times$ 100 & 0.00015 & 23.6 $\pm$ 0.5 & 30 $\pm$ 20 & 22.2 $\pm$ 1 & 26 $\pm$ 2 & 0 \\
5 $\times$ 100 & 0.000175 & 22 $\pm$ 1 & 1000 $\pm$ 1000 & 26.5 $\pm$ 0.5 & 32.4 $\pm$ 0.7 & 0 \\
5 $\times$ 100 & 0.0002 & 22 $\pm$ 1 & 1000 $\pm$ 2000 & 31 $\pm$ 1 & 39 $\pm$ 1 & 1 \\
\bottomrule
\end{tabular}
\end{sc}
\end{small}
\end{center}
\vskip -0.1in
\end{table}

\begin{table}[!htb]
\caption{\textbf{CIFAR100 Classifiers:} Compression results with fixed height and varying width.}
\label{tab:cifar100_width}
\vskip 0.15in
\begin{center}
\begin{small}
\begin{sc}
\resizebox{\columnwidth}{!}{
\begin{tabular}{ccccccc}
\toprule
&&& Compression & \multicolumn{2}{c}{\% Removed} & Timed \\
Architecture & $\lOne$ & Accuracy (\%) & Runtime (s) & Neurons & Connections & Out \\
\midrule
2 $\times$ 100 & 0 & 25.2 $\pm$ 0.2 & 13 $\pm$ 1 & 0 $\pm$ 0 & 0 $\pm$ 0 & 0 \\
2 $\times$ 100 & 0.000025 & 24.8 $\pm$ 0.4 & 13 $\pm$ 2 & 0 $\pm$ 0 & 0 $\pm$ 0 & 1 \\
2 $\times$ 100 & 0.00005 & 24 $\pm$ 0.7 & 11.4 $\pm$ 0.4 & 36 $\pm$ 4 & 36 $\pm$ 4 & 1 \\
2 $\times$ 100 & 0.000075 & 23.7 & 10.4 & 42 & 42 & 4 \\
2 $\times$ 100 & 0.0001 & 23.4 $\pm$ 0.6 & 10.3 $\pm$ 0.1 & 42 $\pm$ 1 & 43 $\pm$ 1 & 1 \\
2 $\times$ 100 & 0.000125 & 22 $\pm$ 2 & 10.453 $\pm$ 0.004 & 48 $\pm$ 1 & 48 $\pm$ 2 & 3 \\
2 $\times$ 100 & 0.00015 & 22.4 $\pm$ 0.6 & 10.8 $\pm$ 1 & 48 $\pm$ 3 & 48 $\pm$ 3 & 1 \\
2 $\times$ 100 & 0.000175 & 21.5 $\pm$ 0.5 & 10.8 $\pm$ 0.3 & 47.9 $\pm$ 0.2 & 48.7 $\pm$ 0.4 & 1 \\
2 $\times$ 100 & 0.0002 & 21 $\pm$ 1 & 10.6 $\pm$ 0.3 & 51 $\pm$ 2 & 52 $\pm$ 2 & 0 \\
2 $\times$ 100 & 0.000225 & 21.2 $\pm$ 0.4 & 11 $\pm$ 0.7 & 51 $\pm$ 2 & 52 $\pm$ 2 & 0 \\
2 $\times$ 100 & 0.00025 & 21 $\pm$ 2 & 10.6 $\pm$ 0.5 & 50 $\pm$ 3 & 52 $\pm$ 4 & 0 \\
2 $\times$ 100 & 0.000275 & 20.7 $\pm$ 0.8 & 10.4 $\pm$ 0.1 & 52 $\pm$ 2 & 54 $\pm$ 2 & 0 \\
2 $\times$ 100 & 0.0003 & 19 $\pm$ 1 & 10.6 $\pm$ 0.2 & 53 $\pm$ 2 & 55 $\pm$ 2 & 0 \\
2 $\times$ 100 & 0.000325 & 19 $\pm$ 1 & 10.7 $\pm$ 0.7 & 53 $\pm$ 4 & 55 $\pm$ 4 & 0 \\
2 $\times$ 100 & 0.00035 & 19.2 $\pm$ 0.9 & 11 $\pm$ 1 & 53 $\pm$ 2 & 55 $\pm$ 1 & 0 \\
2 $\times$ 100 & 0.000375 & 19.4 $\pm$ 0.5 & 10.5 $\pm$ 0.4 & 54 $\pm$ 2 & 56 $\pm$ 2 & 0 \\
2 $\times$ 100 & 0.0004 & 19 $\pm$ 0.5 & 10.5 $\pm$ 0.3 & 53 $\pm$ 3 & 56 $\pm$ 3 & 0 \\
\midrule
\midrule
2 $\times$ 200 & 0 & 28.2 $\pm$ 0.3 & 25 $\pm$ 3 & 0 $\pm$ 0 & 0 $\pm$ 0 & 0 \\
2 $\times$ 200 & 0.000025 & 28.5 & 29.4 & 0 & 0 & 4 \\
2 $\times$ 200 & 0.00005 & 28.1 $\pm$ 0.4 & 27 $\pm$ 7 & 31 $\pm$ 2 & 42 $\pm$ 3 & 0 \\
2 $\times$ 200 & 0.000075 & 27.6 $\pm$ 0.3 & 40 $\pm$ 10 & 36 $\pm$ 1 & 48 $\pm$ 1 & 0 \\
2 $\times$ 200 & 0.0001 & 26.9 $\pm$ 0.3 & 27 $\pm$ 9 & 40 $\pm$ 1 & 52 $\pm$ 1 & 0 \\
2 $\times$ 200 & 0.000125 & 26.1 $\pm$ 0.3 & 20.8 $\pm$ 0.5 & 44 $\pm$ 2 & 57 $\pm$ 2 & 0 \\
2 $\times$ 200 & 0.00015 & 25.7 $\pm$ 0.2 & 21 $\pm$ 1 & 46 $\pm$ 2 & 58 $\pm$ 2 & 0 \\
2 $\times$ 200 & 0.000175 & 25 $\pm$ 0.3 & 21.1 $\pm$ 0.8 & 48 $\pm$ 2 & 60 $\pm$ 2 & 0 \\
2 $\times$ 200 & 0.0002 & 24.2 $\pm$ 0.4 & 21.2 $\pm$ 0.6 & 49.1 $\pm$ 0.9 & 61.6 $\pm$ 0.9 & 1 \\
\midrule
2 $\times$ 400 & 0 & 30.2 $\pm$ 0.2 & 46.2 $\pm$ 0.8 & 0 $\pm$ 0 & 0 $\pm$ 0 & 1 \\
2 $\times$ 400 & 0.000025 & 30.71 $\pm$ 0.04 & 51 $\pm$ 7 & 0 $\pm$ 0 & 0 $\pm$ 0 & 2 \\
2 $\times$ 400 & 0.00005 & 30.2 $\pm$ 0.3 & 60 $\pm$ 10 & 26.5 $\pm$ 0.8 & 42 $\pm$ 1 & 1 \\
2 $\times$ 400 & 0.000075 & 29.13 $\pm$ 0.09 & 49 $\pm$ 5 & 33 $\pm$ 2 & 51 $\pm$ 3 & 2 \\
2 $\times$ 400 & 0.0001 & 28 $\pm$ 0.4 & 51 $\pm$ 7 & 38.3 $\pm$ 0.7 & 56.8 $\pm$ 0.9 & 1 \\
2 $\times$ 400 & 0.000125 & 26.8 $\pm$ 0.4 & 44 $\pm$ 1 & 43 $\pm$ 1 & 62 $\pm$ 2 & 1 \\
2 $\times$ 400 & 0.00015 & 25.9 $\pm$ 0.3 & 47 $\pm$ 3 & 45 $\pm$ 2 & 64 $\pm$ 2 & 1 \\
2 $\times$ 400 & 0.000175 & 25 $\pm$ 0.2 & 44 $\pm$ 3 & 47 $\pm$ 1 & 66 $\pm$ 2 & 0 \\
2 $\times$ 400 & 0.0002 & 24.2 $\pm$ 0.1 & 44 $\pm$ 2 & 48 $\pm$ 2 & 66 $\pm$ 2 & 0 \\
\midrule
2 $\times$ 800 & 0 & 31.32 $\pm$ 0.09 & 100 $\pm$ 20 & 0 $\pm$ 0 & 0 $\pm$ 0 & 2 \\
2 $\times$ 800 & 0.00005 & 30.9 $\pm$ 0.3 & 300 $\pm$ 100 & 21.4 $\pm$ 0.8 & 38 $\pm$ 1 & 0 \\
2 $\times$ 800 & 0.000075 & 29.4 $\pm$ 0.2 & 200 $\pm$ 100 & 32.5 $\pm$ 0.5 & 52.2 $\pm$ 0.8 & 0 \\
2 $\times$ 800 & 0.0001 & 27.8 $\pm$ 0.3 & 97 $\pm$ 5 & 39.1 $\pm$ 0.7 & 60 $\pm$ 0.8 & 0 \\
2 $\times$ 800 & 0.000125 & 26.7 $\pm$ 0.2 & 2000 $\pm$ 4000 & 41 $\pm$ 1 & 62 $\pm$ 1 & 0 \\
2 $\times$ 800 & 0.00015 & 25.8 $\pm$ 0.2 & 98 $\pm$ 5 & 42 $\pm$ 1 & 64 $\pm$ 2 & 0 \\
2 $\times$ 800 & 0.000175 & 24.6 $\pm$ 0.2 & 200 $\pm$ 100 & 44 $\pm$ 2 & 65 $\pm$ 2 & 0 \\
2 $\times$ 800 & 0.0002 & 23.6 $\pm$ 0.5 & 110 $\pm$ 10 & 44.4 $\pm$ 1 & 66 $\pm$ 1 & 0 \\
\bottomrule
\end{tabular}
}
\end{sc}
\end{small}
\end{center}
\vskip -0.1in
\end{table}

\textbf{Runtime Comparison with SoTA~~}
\myReferFigure{fig:cifar_hundred_runtime_with_reg} shows the comparison of runtime of the proposed method and the baseline with the strength of $\lOne$ regularization on the CIFAR-100 classifiers. We observe that the new method presents a median gain of $\mathbf{137}$ times in performance.
\begin{figure}[!t]
    \centering
    \begin{subfigure}[t]{0.45\linewidth}
        \runtimeRegularizationPlot{0.7}{2}{300000}{0, 0.00005, 0.0001, 0.00015, 0.0002}{
            \addLinePlotWithErrorBar {mpl_red }{2pt}{data/cifar_hundred/cifar_hundred_figure1_sheet.txt}{100-100}{smooth}{runtime}{runtime-std}
            \addlegendentry{$2 \times 100$}
            
            \addLinePlotWithErrorBar {mpl_orange }{2pt}{data/cifar_hundred/cifar_hundred_figure1_sheet.txt}{200-200}{smooth}{runtime}{runtime-std}
            \addlegendentry{$2 \times 200$}
            
            \addLinePlotWithErrorBar {mpl_green  }{2pt}{data/cifar_hundred/cifar_hundred_figure1_sheet.txt}{400-400}{smooth}{runtime}{runtime-std}
            \addlegendentry{$2 \times 400$}
            
            \addLinePlotWithErrorBar {mpl_blue}{2pt}{data/cifar_hundred/cifar_hundred_figure1_sheet.txt}{800-800}{smooth}{runtime}{runtime-std}
            \addlegendentry{$2 \times 800$}

            \addLinePlotWithErrorBar {mpl_red }{2pt}{data/cifar_hundred/cifar_hundred_figure1_sheet.txt}{100-100}{dotted}{runtime-old}{runtime-old-std}
            \addLinePlotWithErrorBar {mpl_orange }{2pt}{data/cifar_hundred/cifar_hundred_figure1_sheet.txt}{200-200}{dotted}{runtime-old}{runtime-old-std}
            \addLinePlotWithErrorBar {mpl_green  }{2pt}{data/cifar_hundred/cifar_hundred_figure1_sheet.txt}{400-400}{dotted}{runtime-old}{runtime-old-std}
            \addLinePlotWithErrorBar {mpl_blue}{1pt}{data/cifar_hundred/cifar_hundred_figure1_sheet.txt}{800-800}{dotted}{runtime-old}{runtime-old-std}

        }
        \caption{With width}
    \end{subfigure}
    \hfill
    \begin{subfigure}[t]{0.45\linewidth}
        \runtimeRegularizationPlot{0.7}{2}{300000}{0, 0.00005, 0.0001, 0.00015, 0.0002}{
            \addLinePlotWithErrorBar {mpl_red }{2pt}{data/cifar_hundred/cifar_hundred_figure1_sheet.txt}{100-100}{smooth}{runtime}{runtime-std}
            \addlegendentry{$2 \times 100$}
            \addLinePlotWithErrorBar {mpl_orange}{2pt}{data/cifar_hundred/cifar_hundred_figure1_sheet.txt}{100-100-100}{smooth}{runtime}{runtime-std}
            \addlegendentry{$3 \times 100$}
    
            \addLinePlotWithErrorBar {mpl_red }{2pt}{data/cifar_hundred/cifar_hundred_figure1_sheet.txt}{100-100}{dotted}{runtime-old}{runtime-old-std}
            \addLinePlotWithErrorBar {mpl_orange}{2pt}{data/cifar_hundred/cifar_hundred_figure1_sheet.txt}{100-100-100}{dotted}{runtime-old}{runtime-old-std}
            
        }
        \caption{With depth}
    \end{subfigure}
    \caption{
    \textbf{CIFAR-100 Classifiers: Comparison of runtimes} for proposed method (solid) and baseline (dashed) with the strength of regularization to identify stable neurons: (a) with increasing width (b) with increasing depth. We report the average and the standard deviation of the runtime of models with five different initialization for each regularization. Note that the y-axis is in the log scale.
    The median speedup is $\mathbf{137}$ times. 
    }
    \label{fig:cifar_hundred_runtime_with_reg}
\end{figure}

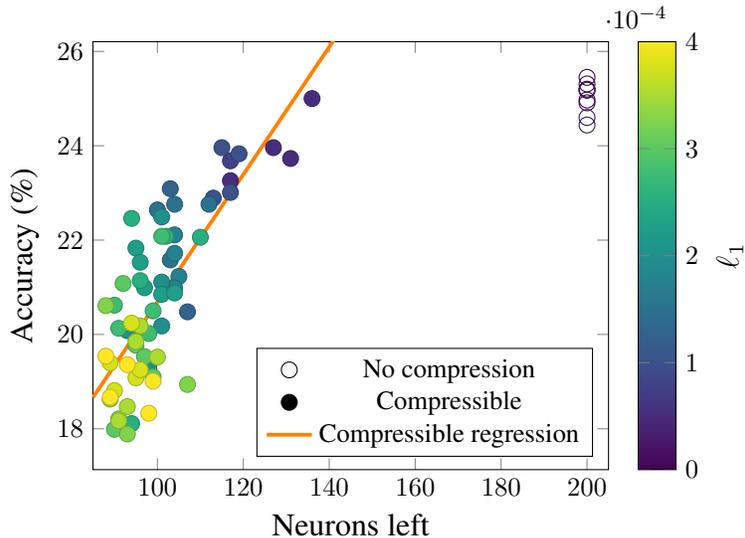
\begin{figure}[!htb]
    \centering
        \begin{tikzpicture}
        \begin{axis}[
            legend pos=south east,
            colorbar,
            point meta=explicit,
            point meta min=0,
            point meta max=0.0004,
            cycle list name=mark list,
            colormap name=viridis,
            xlabel={\large{Neurons left}},
            ylabel={\large{Accuracy (\%)}},
            colorbar style={ylabel={\large$\lOne$}},
            xmin=85,
            xmax=205,
        ]
        \addlegendentry{No compression}
        \addplot+[
            only marks,
            scatter,
            mark=o,
            mark size=3pt]
        table[meta=reg,x=nodes,y=accuracy]
        {data/cifar100_uncompressed_regression.txt};
        \addplot+[
            only marks,
            scatter,
            mark=*,
            mark size=3pt]
        table[meta=reg,x=nodes,y=accuracy]
        {data/cifar100_compressed_regression.txt};
        \addlegendentry{Compressible}
        \draw[color=orange,line width=0.5mm]
            (axis cs:85,18.6638) -- 
            (axis cs:205,34.8758);
        \addlegendimage{line width=0.5mm,color=orange}
        \addlegendentry{Compressible regression}
        \end{axis}
        \end{tikzpicture}
    \caption{\textbf{Relationship between size of compressed neural network and accuracy on $2 \times 100$ CIFAR-100 classifiers.} The coefficient of determination (R$^2$) for the linear regression obtained for accuracy based on neurons left for compressible networks is $61$\%.}
    \label{fig:mnist_accuracy}
\end{figure}

\clearpage
\subsection{Extensions to CNNs: CIFAR-10 LeNet Classifiers}\label{ap:cnn}

We also test our approach with the LeNet~\cite{lecun1998mnist} architecture on CIFAR-10 using the preprocessing step as a predictor of neuron stability to make it more scalable. 
We note that in this case we would only use our method as a sparsification technique to mask stably inactive zeros due to parameter sharing.

When no regularization is used and the test accuracy on CIFAR-10 is around $68.7$\% before pruning, we find that an average of $10.98$\% of the stably inactive neurons can be masked as $0$. 
With an $\lOne$ regularization of $0.000175$, test accuracy on CIFAR-10 is around $70.02$\% before pruning while an average of $11.86$\% of the stably inactive neurons can be masked as $0$. 
In comparison to the case of MLPs, we observe more variability on the number of stable neurons across networks trained with the same amount of regularization, which we believe is due to weight sharing. 
Similar to the case of MLPs, the proportion of neurons that are stable in the training set but not stable in the test set is relatively small: $1.06$\% on average. 
Moreover, we observe that pruning those extra neurons has a zero net effect on accuracy for regularization values in the interval $[0, 0.0003]$.

On a final note, we emphasize that masking $10$\% of the neurons is more strict than masking $10$\% of the parameters as done for lossy compression. 
Furthermore, masking $10$\% of the neurons does not prevent someone from sparsifying the CNN even further: our method merely identifies a set of neurons---and corresponding parameters---that can be ignored for not being relevant. 
We believe that our method could be used in conjunction with conventional sparsification techniques in order to decompose the pruning operations of those into a lossless and a lossy component.

\section{Extensions to Data and Batch Normalization}\label{ap:normalization}

Normalization layers, specially Batch Normalization~\cite{ioffe2015batch}, are present in almost every modern neural network~\cite{he2016deep}.
We now show how to extend our approach to these layers.

\paragraph{Data Normalization}
Data normalization transforms the input $x$ as
\begin{align}
    \text{Norm}(x) &= \frac{x-\mu}{\sigma},
\end{align}
where $(\mu, \sigma)$ correspond to the mean and standard deviation of the data, respectively.

Since, we assume the image pixels to lie in the range $[0, 1]$, the data normalization layer brings the image pixels in the range $\left[-\frac{\mu}{\sigma}, \frac{1-\mu}{\sigma} \right]$. 
Thus, we incorporate data normalization in our approach by adjusting the input bounds using the mean and standard deviation parameters.
Hence, we replace the constraint $x \in [0, 1]$ with the new constraint $x \in \left[-\frac{\mu}{\sigma}, \frac{1-\mu}{\sigma} \right]$.

\paragraph{Batch Normalization}
Batch Normalization~(BN)~\cite{ioffe2015batch} corresponds to applying the affine transformation to the input $x$ as 
\begin{align}
    \text{BN}(x) &= \gamma \left( \frac{x-\meanbatch}{\sqrt{\varbatch + \epsilon}} \right) + \beta,
\end{align}
where $(\meanbatch, \varbatch)$ are the mean and variance (the mini-batch statistics) of the data, $(\gamma, \beta)$ are the trainable parameters, and $\epsilon$ is a small constant to avoid division by zero.  

For lossless compression, we run the MILP solver after the training of the neural network completes. 
Thus, BN mini-batch statistics are frozen (do not update) while running MILP, and BN only serves to scale the layer input. 
If the layer input before the BN layer is in the range $[h_{min}, h_{max}]$, the BN layer brings these input in the range
$\left[\gamma \left( \dfrac{h_{min}-\meanbatch}{\sqrt{\varbatch + \epsilon}} \right) + \beta, \gamma \left( \dfrac{h_{max}-\meanbatch}{\sqrt{\varbatch + \epsilon}} \right) + \beta \right]$.
Thus, BN does not introduce any extra constraint for the MILP formulation. 

We end this discussion on a final note. Although BN in inference does an affine transform of the input, BN in inference is different from the fully connected layer. 
BN in inference transforms the inputs individually without taking contributions from other inputs into account. 
On the other hand, a fully connected layer does an affine transform while taking the contributions of all inputs into account. 

\end{document}